\documentclass[twocolumn]{autart}    
\usepackage{amsfonts}
\usepackage[table]{xcolor}
\usepackage{graphicx}          
\usepackage{epstopdf}
\usepackage{amsmath}
\usepackage{amssymb}
\usepackage{bm}
\usepackage[numbers,sort&compress]{natbib}
\usepackage{algpseudocode,algorithm}
\usepackage{multirow}
\usepackage{subfigure}
\usepackage{enumerate}
\usepackage{textcomp}
\usepackage{mathrsfs} 
\usepackage{setspace}
\usepackage{nomencl}
\usepackage{threeparttable}
\usepackage{comment}
\usepackage{appendix}
\usepackage{wrapfig}
\usepackage[colorlinks=true, linkcolor=blue, urlcolor=red, citecolor=black]{hyperref}

\graphicspath{{Graphics/}}
\newtheorem{lemma}{\bf Lemma}

\newtheorem{theorem}{\bf Theorem}
\newtheorem{remark}{\bf Remark}

\newtheorem{property}{Property}

\begin{document}
	\setlength{\abovedisplayskip}{3pt}
	\setlength{\belowdisplayskip}{3pt}
	\begin{frontmatter}
		
		\title{Hierarchical Variational Kalman Filtering}
		
		\thanks[footnoteinfo]{Shilei Li and Dawei Shi are with School of Automation, Beijing Institute of Technology, Beijing 100081, China (e-mail: shileili@bit.edu.cn, daweishi@bit.edu.cn). Wei Zheng is with Shenzhen Youjia Innovation Technology Co., Ltd., Shenzhen, China (e-mail: zhengwei@minieye.cc). Ling Shi is with the Department of Electronic and Computer Engineering, The Hong Kong University of Science and Technology, Kowloon, Hong Kong (e-mail: eesling@ust.hk). (Corresponding author: Dawei Shi.)}
		\author[ad1]{Shilei Li}\ead{shileili@bit.edu.cn}, 
		\author[ad1]{Dawei Shi}\ead{daweishi@bit.edu.cn},
		\author[ad2]{Wei Zheng}\ead{zhengwei@minieye.cc},
		\author[ad3]{Ling Shi}\ead{eesling@ust.hk}
		
		\address[ad1]{School of Automation, Beijing Institute of Technology, Beijing 100081, China}
		\address[ad2]{Shenzhen Youjia Innovation Technology Co., Ltd., Shenzhen, China}
		\address[ad3]{Department of Electronic and Computer Engineering, The Hong Kong University of Science and Technology, Clear Water Bay, Kowloon, Hong Kong, China}
		
		\begin{keyword}
			Variational inference, Marginalized maximum a posteriori, Rank-deficient covariance tracking, Kalman filtering
		\end{keyword}
		
		\begin{abstract}
			Traditional variational Kalman filtering with unknown noise statistics suffers from inconsistent process covariance estimation and slow convergence speed, limiting its practical utility. To address these issues, we introduce a surrogate variable representing the process-noise-free state, which enables explicit modeling and inference of process noise statistics. In addition, we reformulate the conventional coordinate ascent variation inference (CAVI) as a marginalized maximum a posteriori problem, followed by a single-step hyperparameter fitting. This reformulation obviates the need for multiple inner iterations inherent to CAVI and decouples the design of the covariance tracking filters. Consequently, this architecture permits the deployment of higher-order filters for covariance tracking and enables sliding-window hyperparameter estimation. Notably, when this window encompasses all historical data, the covariance tracking estimator intrinsically operates as a zero-phase filter. Numerical simulations validate the theoretical framework, demonstrating the enhanced convergence speed and superior estimation accuracy compared with existing methods.
		\end{abstract}
	\end{frontmatter}
	\section{Introduction}
	\label{sec:introduction}
	{Inferring} the hidden state of dynamical systems in real time from the noisy measurements forms the backbone of modern state estimation and control, with many applications including the simultaneous localization and mapping in robotics~\cite{a1}, the camera motion estimation in signal processing~\cite{a2}, the weather prediction in meteorology~\cite{a3}, the target tracking in control~\cite {a4}, to name only a few. The pioneering works of this field trace back to a paper co-authored by Rudolf Kalman and Richard S. Bucy~\cite{a5}, which gives an unbiased minimum estimator under a linear model with Gaussian noise. It is further extended by Schmidt’s group for nonlinear systems with application to solving the Apollo guidance and navigation~\cite{a6}. After that, many variants of the Kalman filter (KF), e.g., UKF~\cite{a7}, CKF~\cite{a8}, correntropy-based KF~\cite{a9}, consensus-based filter~\cite{b1}, have been developed with special emphasis on the nonlinear systems, non-Gaussian noise, and distributed systems. 
	
	Many estimators assume exact knowledge of the noise statistics. However, this assumption is rarely satisfied in practical applications. One important reason is that an exact system model is usually unattainable, and approximations are inevitable. In such cases, unmodeled dynamics can be interpreted as time-varying process or measurement noises.  Modeling errors may also arise from the linearization of nonlinear systems. For example, in the extended Kalman filter (EKF), the system is linearized at the \emph{a priori} estimate of the state using a first-order approximation, implying that the resulting modeling errors depend on the degree of system nonlinearity. If not properly handled, such errors may induce a positive feedback loop and lead to filter divergence~\cite{a11}. 
	
	Operational conditions and external environments also contribute to time-varying noise statistics. For instance, in quadruped robot localization, contact-related noise varies as the robot traverses different terrains~\cite{a12}. In Global Positioning System (GPS) navigation, measurement noise statistics change with satellite visibility and are further affected by multipath and non-line-of-sight (NLOS) propagation~\cite{a13}. Moreover, false data injection (FDI) and denial-of-service (DoS) attacks can also be modeled as switched, or time-varying noise statistics~\cite{a14,a15}, which has attracted increasing attention in cyber-physical system security.
	
	There exist several lines of research on state estimation with time-varying noise statistics, which can be broadly classified into model-based and learning-based approaches. The former category includes moment matching~\cite{a21}, innovation-based methods~\cite{a16}, robust loss–based estimators~\cite{c1,a9}, interacting multiple-model (IMM) approaches~\cite{a17}, and variational Bayesian (VB) inference~\cite{a4,a18}. Covariance matching and innovation-based methods typically rely on second-order statistics of the residuals, which require a sufficiently large sliding window to obtain reliable covariance estimates and may suffer from a risk of divergence~\cite{a22}. Robust loss–induced estimators offer guaranteed convergence; however, by discarding historical noise covariance information, they exhibit limited adaptivity. The IMM-based approach employs a bank of Kalman filters to approximate the unknown noise covariance, but suffers a significant computational burden and becomes impractical for high-dimensional systems, as the number of required filters grows exponentially with the system dimension~\cite{a17}. The VB-based approach, originally proposed by S{\"a}rkk{\"a} and Nummenmaa~\cite{a4}, jointly estimates the state and noise variance by minimizing the Kullback–Leibler divergence between the true posterior and a designated distribution. It achieves performance comparable to that of IMM-based methods while significantly reducing the computational cost, and is favored by many researches. 
	
	Unlike purely model-based methods, KalmanNet~\cite{a25} integrates the KF framework with Gated Recurrent Units (GRUs) to alleviate model mismatch by learning a specialized Kalman gain. This architecture has been further extended in~\cite{a26} and~\cite{a27}, where Transformers and deep neural networks (DNNs) replace the GRUs to offer alternative performance trade-offs. More recently, Buchnik et al.~\cite{buchnik2023latent} introduced latent-KalmanNet, which operates in a latent space to jointly learn feature representations and filtering operations, achieving notable performance gains in visual target tracking. Dahan et al.~\cite{dahan2025bayesian} further augmented it by incorporating uncertainty quantification through a Bayesian Kalman architecture. Another emerging line of research integrates generative modeling with state estimation, leveraging normalizing flows or diffusion models~\cite{wan2026diffpf,chen2024normalizing} to enable high-dimensional, multimodal estimation and extend standard Gaussian posteriors to much more complex distributions.

	It is worth noting that while learning-based methods raise the performance upper bound compared to conventional techniques, they rely heavily on extensive ground-truth state trajectories for training, which are frequently unavailable in practice~\cite{a25,a26,a27,buchnik2023latent}. Furthermore, their generalization capabilities are often compromised when they encounter unseen operational conditions. Finally, learning-based approaches, particularly generative models, inherently operate as black boxes, introducing potential safety risks and making deployment on embedded devices difficult. Alternatively, the VB-based method solves a similar problem by explicitly parameterizing the noise statistics and inferring their posterior distributions via variational Bayesian approximation, which is tractable and explainable. However, existing VB-based methods have some limitations: (1) They are sensitive to outliers and the estimator's performance degenerates significantly with heavy-tailed noises~\cite{a34}; (2) The coordinate ascent variational inference, as a standard tool for variational optimization, converges in a zig-zag way, which is computationally inefficient. (3) Heuristic noise variance dynamics is widely used but without in-depth analysis on its tracking speed~\cite{a4,a19,a20,a28,a29,a30,a31}; (4) The existing solutions cannot well accommodate the time-varying rank-deficient process. (5) The VB-based methods can only track slow-varying noise variance and possess a delay in fast-varying noise statistics~\cite{a28,a33}. These limitations restrict the application of VB-based estimators.
	
	To address the above-mentioned challenges, in this work, we decouple the conventional VB into a sequential marginalized MAP (MMAP)  and hyper-parameter update. Firstly, we reformulate the conventional state-space model as a lifted regression framework, which allows us to separate the process noise and perform inference channel by channel, giving flexibility in handling rank-deficient process covariance. Secondly, we decompose the VB inference as a two-stage problem: MMAP estimation and one single step hyper-parameter update. This operation can not only save the computation resources, but also decouple the state update and variational hyper-parameter update, given a room in handling variational-robust problems~\cite{a24}. Thirdly, we compensate for the variational delay by exploiting the historical errors by adding regulations or considering higher order filters. In the case of covering all historical data, the covariance tracking effectively functions as a zero-phase filter. The contributions of this work are summarized as follows:
	\begin{itemize}
	\item We formulate the state estimation problem as a lifted whitened regression problem and provide the identifiability condition (\textbf{Lemma \ref{identlemma}}), which allows us to treat the process and measurement indiscriminately. In addition, we separate the conventional VB estimation as a two-stage problem: MMAP state estimation and hyper-parameter update, which is computationally efficient and enables robustness and adaptiveness to be achieved sequentially. 
	\item We solve the MMAP using an iterative reweighted least squares (IRLS) method. Meanwhile, we provide the convergence proof of the lifted equation under a class of robust losses in \textbf{Theorem \ref{convRWLS}}. 
	\item We provide a detailed analysis of covariance-tracking speed in \textbf{Theorem \ref{prop:refined_estimator}} and compensate for the delay by exploiting historical information in Algorithms \ref{HOHVKF} and \ref{Alg:sliding_window}. Simulations verify the effectiveness of the proposed method.
	\end{itemize}
	
	The rest of this paper is organized as follows. Section II provides preliminaries and problem formulation.  Section III gives the detailed methodologies. Section IV provides some simulations. Section V concludes the paper.
	
	Notations: $A^T$ and $A^{-1}$ denote the transpose and inverse of matrix $A$. $X \succ 0$ ($X \succcurlyeq 0$) denotes a positive (or semi-positive) definite matrix, and $I_n$ is a $n \times n$ identity matrix. $\mathcal{N}(\mu, \Sigma)$, $\operatorname{St}(0, \tau^2, \nu)$, and $\operatorname{IG}(\alpha, \beta)$ represent the Gaussian, zero-mean Student's $t$, and inverse-gamma distributions, respectively. $\mathbb{E}[\cdot]$ and $\mathbb{E}_q[\cdot]$ denote expectations (standard, and with respect to distribution $q$). Finally, $[v]_i$ and $[M]_{ii}$ represent the $i$-th element of vector $v$ and $i$-th diagonal element of matrix $M$.

	\section{Preliminaries and Problem Formulation}
	We first provide some preliminaries and then give the problem formulation.
	\subsection{Hierarchical Formulation of the Student's $t$-Distribution}
	A random variable $x \sim \operatorname{St}(\mu, \tau^2, \nu)$ with degrees of freedom (DOF) $\nu$ has the probability density function:
	\begin{equation}
		p(x) = \frac{\Gamma((\nu+1)/2)}{\Gamma(\nu/2)\sqrt{\pi\nu\tau^2}} \left( 1 + \frac{(x-\mu)^2}{\nu\tau^2} \right)^{-(\nu+1)/2}.
		\label{stdis}
	\end{equation}
	This heavy-tailed distribution can be equivalently formulated as a Gaussian-scale mixture:
	\begin{equation}
		x \sim \mathcal{N}(\mu, \lambda), \quad \lambda \sim \operatorname{IG}\left(\frac{\nu}{2}, \frac{\nu\tau^2}{2}\right)
		\label{hdis}
	\end{equation}
	where \begin{equation}\nonumber
		\operatorname{IG}\left(\lambda \mid \frac{\nu}{2}, \frac{\nu\tau^2}{2}\right) = \frac{\left(\frac{\nu\tau^2}{2}\right)^{\frac{\nu}{2}}}{\Gamma\left(\frac{\nu}{2}\right)} \lambda^{-\left(\frac{\nu}{2}+1\right)} \exp\left(-\frac{\nu\tau^2}{2\lambda}\right)
	\end{equation} 
	is the inverse-gamma distribution. 
	\begin{remark}
		The symbol $\nu$ and $\tau^2$ have a clear interpretation. Specifically, $\nu$ controls the tail heaviness of the induced Student's $t$-distribution and dictates the estimator's robustness against outliers. Meanwhile, $\tau^2$ serves as the nominal scale parameter, representing the latent covariance of the Gaussian distribution. 
		\label{remark1}
	\end{remark}
	
	By taking the negative logarithm on the right side of \eqref{stdis} and defining $e=x-\mu$, we obtain (ignoring the constant terms):
	\begin{equation}
		\mathcal{L}_{st}^{*}(e) =  \frac{\nu+1}{2}\log\left(1+ \frac{e^2}{\nu\tau^2}\right).
		\label{stloss}
	\end{equation}
	\begin{property}
		As $\nu \to \infty$, one has $\mathcal{L}_{st}=\frac{1}{2}\frac{e^2}{\tau^2}$.
	\end{property}
	This can be proved by taking limit on \eqref{stloss}.
	\begin{property}[\cite{a24}] As $\nu \to \infty$, one has $\operatorname{St}(\mu, \tau^2, \nu) \to \mathcal{N} (\mu, \tau^2)$ where the latent $\lambda$ distribution degenerates to a shifted Dirac delta function $\delta(\lambda-\tau^2)$.
		\label{pro2}
	\end{property}
	\begin{property}[\cite{a24}]
		$\mathcal{L}_{st}$ is convex within $[-\sqrt{\nu}\tau,\sqrt{\nu}\tau]$, and is concave in other regions.
	\end{property}
	\begin{property}
		By denoting $u \triangleq e^2$, the function $\mathcal{L}_{st}(u)= \frac{\nu+1}{2}\log\left(1+ \frac{u}{\nu\tau^2}\right)$ is concave.
	\end{property}
	This can be proved by verifying $\frac{d^2}{du^2}\mathcal{L}_{st}(u) < 0$. 
	\begin{prop}[Half-Quadratic Duality~\cite{nikolova2005analysis}]
		\label{prop_hq}
		Let $\rho(e)$ be a robust loss function such that $\phi(u) = \rho(\sqrt{u})$ is concave, continuously differentiable, and monotonically increasing on $u \in [0, \infty)$. There exists a dual potential function $\psi(\lambda)$ such that the non-convex loss can be expressed as the infimum of an augmented function:
		\begin{equation}
			\begin{aligned}
				\rho(e) = \inf_{\lambda > 0} \left[ \frac{e^2}{2\lambda} + \psi(\lambda) \right]
			\end{aligned}
		\end{equation}
		where $\lambda$ acts as an auxiliary weight variable. The minimum is achieved when $\lambda = \frac{e}{\rho'(e)}$, mapping the robust optimization problem to an iteratively reweighted least squares (IRLS) sequence.
	\end{prop}
	\begin{lemma}[Mean Equivalence]
		\label{equivalent}
		Let $p(x)$ be a continuously symmetric and unimodal probability density function with its unique mode at $x^*$. If we approximate $p(x)$ using a Gaussian distribution $q(x) = \mathcal{N}(x; \mu, \Sigma)$, then maximizing the ELBO, or equivalently minimizing the KL divergence $\mathrm{KL}(q(x) || p(x))$, with respect to the variational mean $\mu$ yields:$$\mu^* = \arg\min_{\mu} \mathrm{KL}(q(x) || p(x)) = x^*$$where $x^* = \arg\max_x p(x)$ is the MAP estimate.
	\end{lemma}
	\begin{pf}
		The ELBO to be maximized with respect to $\mu$ and $\Sigma$ is:
		$$\mathcal{L}(\mu, \Sigma) = \mathbb{E}_{q(x)}[\log p(x)] + \frac{1}{2}\log|\Sigma| + c,$$
		where c is a constant. Taking the gradient of the ELBO with respect to $\mu$, and utilizing the identity $\nabla_\mu \mathbb{E}_{q}[\log p(x)] = \mathbb{E}_{q}[\nabla_x \log p(x)]$ (via Bonnet's theorem or the reparameterization trick), we set it to zero, which gives
		$$\nabla_\mu \mathcal{L} = \mathbb{E}_{\mathcal{N}(x; \mu, \Sigma)}[\nabla_x \log p(x)] = 0$$
		Since $p(x)$ is symmetric and unimodal around $x^*$, the score function $\nabla_x \log p(x)$ is an odd function centered at $x^*$. For the expectation of an odd function under a symmetric Gaussian distribution to be strictly zero, the Gaussian distribution must be perfectly centered at the symmetry point of that odd function. Therefore, $\mu^* = x^*$.
	\end{pf}
	\subsection{Problem Formulation} 
	The system dynamics is as follows:
	\begin{equation}
		\begin{aligned}
			x_{k}&=A x_{k-1} + B w_k \\
			y_k &= C x_{k} +D v_k
		\end{aligned}
		\label{sys}
	\end{equation}
	where $A \in \mathbb{R}^{n\times n}$ and $C \in \mathbb{R}^{m \times n}$ are the state-transition and observation matrices, respectively; $B \in \mathbb{R}^{n \times p}$ and $D \in \mathbb{R}^{m \times m}$ map the process and measurement noises to the state and observation. The noise vectors $w_k \in \mathbb{R}^{p}$ and $v_k \in \mathbb{R}^{m}$ are element-wise uncorrelated zero-mean random variables (i.e., $E(w_{i,k} w_{j,k}) = 0$ and $E(v_{i,k} v_{j,k}) = 0$ for $i \neq j$, where $w_{i,k}$ and $v_{i,k}$ denote the $i$-th components of $w_k$ and $v_k$). Without loss of generality, we assume the nominal covariances $E(w_k w_k^{T}) = I_{p}$ and $E(v_k v_k^{T}) = I_{m}$, and define $Q_k = B I_{p} B^{T} \succeq 0$ and $R_k = D I_{m} D^{T} \succ 0$ (implying $D$ is non-singular). The pair $(A, \sqrt{Q_k})$ is controllable and $(A, C)$ is observable. The matrix $B$ has full column rank, i.e., $\text{rank}(B) = p$.
	
	In practical applications, the noise covariance may be time-varying, i.e., $E(w_k w_k^{T}) = Q_{w,k}$ and $E(v_k v_k^{T}) = R_{v,k}$ are time-varying. The associated process and measurement covariances become $Q_k=BQ_{w,k}B^{T}$ and $R_k=DR_{v,k}D^{T}$. The initial state $x_0$ is zero-mean Gaussian with known covariance matrix $P_0$, and is independent of $w_k$ and $v_k$ for all $k>0$. 
	\begin{remark}
		Unlike the simplified noise modeling in ~\cite{a3,a9}, where disturbances are directly added to the state and measurement, we employ a structural noise model. By explicitly introducing matrices $B$ and $D$, this formulation characterizes the underlying physical mechanisms through which stochastic uncertainties propagate into the system dynamics and observation channels.
	\end{remark}
	
	Our purpose is to jointly estimate $\{\hat{x}_k, \hat{Q}_k, \hat{R}_k\}$ under the governing equation \eqref{sys}.

	\section{Main Results}	
	In this section, we formulate the joint estimation problem as a two-stage sequential problem: marginalized maximum a posteriori (MMAP) estimation and one step hyper-parameter update.
	\subsection{The Proposed Method}
	The mean-field variational approximation provides a backbone for joint estimate $\{\hat{x}_k, \hat{Q}_k, \hat{R}_k\}$. However, it should be noted that the state $x_k$ depends on the process noise covariance $Q_k$, which complicates the variational inference process, as the joint posterior cannot be directly factorized into independent components. To address this, Huang et al. perform inference on $P_{k|k-1}$ rather than $Q_k$, which simplifies the problem. However, it overlooks the intrinsic structure of $Q_k$ and produces sub-optimal results. Alternatively, we introduce an auxiliary variable
	$m_k \sim \mathcal{N} (m_k^{-}, \Sigma_k) = \mathcal{N} (A x_{k-1}, AP_{k-1|k-1}A^{T}) $, termed the unperturbed prior state, representing the state excluding process noise. Thus, we have 
	\begin{equation}
		\scalebox{0.9}{$
		\begin{bmatrix}
			m_{k}^{-}\\
			0\\
			y_k
		\end{bmatrix} = \begin{bmatrix}
			I & 0\\
			I& -I\\
			0 & C
		\end{bmatrix}\begin{bmatrix}
			m_k\\
			x_k
		\end{bmatrix}+ \begin{bmatrix}
			p_{k}\\
			Bw_k  \\
			Dv_k		
		\end{bmatrix}$}
		\label{sys1}
	\end{equation} 
	where $p_k$ denotes the virtual noise, and $m_k \sim \mathcal{N}(m_k^{-},\Sigma_{Pk})$ with $m_k^{-} = A x_{k-1}$ and $\Sigma_{Pk} = A P_{k-1|k-1} A^T$. Denote
	\begin{equation}
		\scalebox{0.9}{$
		B_kB_k^{T}=\begin{bmatrix}
			B_pB_p^{T} & 0 & 0\\
			0& BB^{T} & 0\\
			0& 0   &DD^{T}
		\end{bmatrix}=\begin{bmatrix}
			\Sigma_{Pk} & 0 & 0\\
			0& Q_k & 0\\
			0& 0   &R_k
		\end{bmatrix}$}
		\label{Bk}
	\end{equation}
	where $B_p$ is obtained by Cholesky decomposition through $B_pB_p^{T}=\Sigma_{Pk}$. It follows that
	\begin{equation}
		B_k^{-1}=
		\begin{bmatrix}
			B_p^{-1} & 0 & 0\\
			0& B^\dagger & 0\\
			0& 0   &D^{-1}
		\end{bmatrix}
	\end{equation}
	where $B^\dagger \in \mathbb{R}^{p \times n}$ is the Moore-Penrose pseudo-inverse of $B$. Left multiplying $B_k^{-1}$ in both side of \eqref{sys1} gives the following \emph{whitened} regression problem:
	\begin{equation}
		t_k = W_k \chi_k + \xi_k
		\label{wreg}
	\end{equation} 
	where
	\begin{equation}
		\label{tkwk}
				\scalebox{0.9}{$
		\begin{aligned}
			t_k &= B_k^{-1}\begin{bmatrix}
				m_{k}^{-}\\
				0\\
				y_k
			\end{bmatrix}, \quad W_k=B_k^{-1}\begin{bmatrix}
				I & 0\\
				I& -I\\
				0 & C
			\end{bmatrix},\quad \chi_k=\begin{bmatrix}
				m_k\\
				x_k
			\end{bmatrix}\\
			\xi_k&=\begin{bmatrix}
				\xi_p\\
				\xi_w\\
				\xi_v
			\end{bmatrix}=B_k^{-1}
			\begin{bmatrix}
				p_{k}\\
				Bw_k  \\
				Dv_k		
			\end{bmatrix}.
		\end{aligned}$}
	\end{equation}
	Note that the prior state estimate $m_k^{-}$ is treated as a pseudo-measurement in the above equation. In the nominal case, one has $E(\xi_k\xi_k^{T})={I}_{n+p+m}$.
	\begin{remark}
		Reformulating \eqref{sys} as \eqref{wreg} offers several advantages. First, by augmenting the system dimensions, it treats the process and measurement noises in a unified manner. This is beneficial especially when $Q_k$ is rank-deficient (i.e., rank $Q_k \le n$). Second, whitening the noise as $\xi_k$ allows us to treat channels in different ways, providing flexibility in high-dimensional and complex problems. 
	\end{remark}
	
	\begin{lemma}[Identifiability]
		\label{identlemma}
		Consider the whitened regression problem in $\eqref{wreg}$ for the augmented state $\chi_k \in \mathbb{R}^{2n}$. Assume the prior covariance $\Sigma_{Pk}$ and the measurement noise covariance $R_k$ are positive definite, and $B$ has full column rank. The state $\chi_k$ has a unique solution if and only if $\mathcal{N}(B^T) \cap \mathcal{N}(C) = \{0\}$ (or equivalently $\mathrm{rank}\Big(\begin{bmatrix}
			B^{T}\\
			C
		\end{bmatrix}\Big)=n$), where $\mathcal{N}(\cdot)$ denotes the null space of a matrix.
	\end{lemma}
	The proof is provided in Appendix \ref{identproof}. Throughout this work, we restrict our attention to systems that satisfy the condition stated in Lemma \ref{identlemma}.
	
	We define $e_k \triangleq\xi_k=t_k- W_k \chi_k$ and denote
	\begin{equation}
		t_k=\begin{bmatrix}
			t_p\\
			t_w\\
			t_v
		\end{bmatrix}, e_k=\begin{bmatrix}
			e_p\\
			e_w\\
			e_v
		\end{bmatrix}, \Sigma_{\chi}= \begin{bmatrix}
			I_n & 0&0\\
			0& \Sigma_{Qk}&0\\
			0& 0& \Sigma_{Rk}
		\end{bmatrix}
	\end{equation}
	where $\Sigma_{Qk}$ and $\Sigma_{Rk}$ are two diagonal matrices to be estimated, reflecting the characteristic of time-varying $w_k$ and $v_k$. Recall that our purpose is to jointly estimate both the state and associated unknown covariance matrices, we need to assign a prior distribution for $e_w$ and $e_v$ to avoid over-fitting~\cite{c1}. To proceed, we model them as element-wise independent Student’s $t$-distributions as follows:
	\begin{equation}
		\begin{aligned}
			e_w&= [e_{w,1},\ldots,e_{w,i},\ldots,e_{w,p}]^{T},\\
			e_v&= [e_{v,1},\ldots,e_{v,j},\ldots,e_{v,m}]^{T}. 
		\end{aligned}
	\end{equation} 
	Each element in the above vector follows a zero-mean Student’s $t$-distribution, i.e., 
	\begin{align}
		e_{w,i} &\sim \operatorname{St}(e_{w,i} \mid 0, \tau_i^2, \nu_i), \quad i = 1, 2, \ldots, p, \\
		e_{v,j} &\sim \operatorname{St}(e_{v,j} \mid 0, \tau_j^2, \nu_j), \quad j = 1, 2, \ldots, m,
	\end{align}
	where $\nu_i$ and $\nu_j$ denote the degrees of freedom (DOF) controlling the tail heaviness of the corresponding distributions, and $\tau_i^2$ and $\tau_j^2$ are the associated nominal variances. Note that although $e_p \sim \mathcal{N}(0,I_n)$ is Gaussian. Based on Property \ref{pro2}, it can be regarded as a  Student’s $t$-distribution as follows:
	\begin{equation}
		\begin{aligned}
			e_{p,i} &\sim \operatorname{St}(e_{p,i} \mid 0, \tau_i^2, \nu_i), \quad i = 1, 2, \ldots, n, \\
		\end{aligned}
	\end{equation} 
	with $\nu_i \to \infty$. This allows us to treat the different elements of $e_p$ in a uniform way. Applying Bayes' rule alongside the Student's $t$-distribution properties from \eqref{stdis}, the joint posterior can be factorized as follows:
	\begin{equation}
		\begin{aligned}
			p&(\chi_k, \Sigma_{Qk}, \Sigma_{Rk} \mid t_{1:k}) = \frac{p(\chi_k, \Sigma_{Qk}, \Sigma_{Rk}, t_k \mid t_{1:k-1})}{p(t_k \mid t_{1:k-1})} \\
			&\propto p(t_k, \chi_k \mid \Sigma_{Qk}, \Sigma_{Rk}) p(\Sigma_{Qk}) p(\Sigma_{Rk}),
		\end{aligned}
		\label{map}
	\end{equation}
	where $p(t_k \mid t_{1:k-1})$ is the marginal likelihood obtained by integrating out the optimization variables, and thus acts merely as a normalizing constant. Observe that the Gaussian prior $p(\chi_k \mid t_{1:k-1})$ is already equivalently represented as a pseudo-measurement embedded within the first block of the whitened regression model in \eqref{wreg}. Consequently, the prior is inherently absorbed into the augmented likelihood term $p(t_k \mid \chi_k, \Sigma_{Qk}, \Sigma_{Rk})$, removing the need for an explicit, separate prior term to avoid double-counting. By denoting $\Theta=\{\chi_k, \Sigma_{Qk},\Sigma_{Rk}\}=\{\chi_k, \lambda^{wv}\}$, we rewrite the joint posterior in a compact way as
	\begin{equation}
		\begin{aligned}
			p&(\Theta|t_{1:k})\propto p(\Theta,t_k|t_{1:k-1})\\
			=& \mathcal{N}(t_k;W_k\chi_k, \Sigma_{\chi})\prod_{i=1}^{n+p+m}\operatorname{IG}(\lambda_i| \frac{\nu_i}{2},\frac{\nu_i \tau_i^2}{2}).
		\end{aligned}
		\label{map1}
	\end{equation}
	The exact solution for \eqref{map1} is intractable due to the coupling between $\chi_k$ and $\lambda^{wv}$. Thus, we resort to variational Bayesian inference to construct a tractable approximation. Specifically, by the mean-field assumption, we assume the target distribution follows 
	\begin{equation}
		p(\Theta|t_{1:k}) \approx q(\chi_k,\lambda^{wv})=q(\chi_k)\prod_{1}^{p+m}q(\lambda_i^{wv}),
	\end{equation} 
	where $q(\chi_k)$ is a Gaussian distribution and $\lambda_i^{wv}$ is diagonal inverse-Gamma variational factors. Then, we minimize the Kullback-Leibler (KL) divergence between the separable approximation and the true posterior:
	\begin{equation}
		\begin{aligned}
			\mathrm{KL}&(q(\chi_k,\lambda^{wv})||p(\Theta|t_{1:k}))\\
			&=\iint q(\chi_k,\lambda^{wv}) \ln \left( \frac{q(\chi_k,\lambda^{wv})}{p(\Theta|t_{1:k})} \right) d\chi_k d\lambda^{wv},
		\end{aligned}
		\label{kldiv}
	\end{equation}
	which is equivalent to maximizing the following evidence lower bound (ELBO)~\cite{blei2017variational}:
	\begin{equation}
		\begin{aligned}
			\mathcal{L}(q) &= \mathbb{E}_q\left[\log p(t_k \mid \chi_k, \lambda^{wv})\right]+ \mathbb{E}_q\left[\log p(\lambda^{wv})\right]\\ 
			&- \mathbb{E}_q\left[\log q(\chi_k)\right] - \mathbb{E}_q\left[\log q(\lambda^{wv})\right].
		\end{aligned}
		\label{elbo0}
	\end{equation} 
	Maximizing the ELBO with respect to $\chi_k$ and $\lambda^{wv}$ gives:
	\begin{subequations}
		\begin{align}
			\log q(\chi_k) &\propto \left( \mathbb{E}_{q(\lambda^{wv})} \left[ \log p(t_k \mid \chi_k, \lambda^{wv}) \right] \right), \label{vbsub1}\\
			\log q(\lambda^{wv}) &\propto \left( \mathbb{E}_{q(\chi_k)} \left[ \log p(t_k \mid \chi_k, \lambda^{wv}) + \log p(\lambda^{wv}) \right] \right) \label{vbsub2}.
		\end{align}
		\label{vbfull}
	\end{subequations}
	The coupled equations in \eqref{vbfull} can be solved iteratively via coordinate-ascent variational inference (CAVI), see details in Appendix \ref{CAVIalg}. The full iterative procedure is computationally expensive and tightly couples the state estimation with covariance tracking at every iteration. To decouple this and alleviate the computational burden, we resort to a computationally efficient approximation. Initially, we marginalize $\lambda^{wv}$ out from the joint distribution \eqref{map1}, which gives   
	\begin{equation}
		\begin{aligned}
			&p(\chi_k, t_k | t_{1:k-1}) = \int p(\Theta, t_k | t_{1:k-1}) d\lambda_1 \dots d\lambda_{n+p+m} \\
			&\propto\prod_{i=1}^{n+p+m} \left( (t_{k,i} - (W_k\chi_k)_i)^2 + \nu_i\tau_i^2 \right)^{-\frac{\nu_i + 1}{2}},
		\end{aligned}
		\label{mposteriori}
	\end{equation}
	which is a product of Gaussian distributions and heavy-tailed Student's $t$-distributions. Then, the KL divergence is redefined to evaluate the difference exclusively between the approximate posterior $q(\chi_k)$ and the true marginalized posterior $p(\chi_k|t_{1:k})$:
	\begin{equation}
		\begin{aligned}\mathrm{KL}&(q(\chi_k)||p(\chi_k|t_{1:k}))=\int q(\chi_k) \ln \left( \frac{q(\chi_k)}{p(\chi_k|t_{1:k})} \right) d\chi_k.
		\end{aligned}
	\end{equation}
	The associated ELBO is:
	\begin{equation}
		\begin{aligned}
			\mathcal{L}(q(\chi_k)) &= \mathbb{E}_{q(\chi_k)}\left[\log p(\chi_k, t_k \mid t_{1:k-1})\right] \\
			&- \mathbb{E}_{q(\chi_k)}\left[\log q(\chi_k)\right].
		\end{aligned}
		\label{elbo_marginalized}
	\end{equation}
	
	It is worth noting that the marginalized posterior in \eqref{mposteriori} is generally asymmetric. Specifically, by evaluating the third-order derivative tensor of the log-posterior at the MAP estimate $\chi^*$, the local skewness along any direction $\Delta$ is driven by the coefficient $k_i \propto e_i (e_i^2 - 3\nu_i\tau_i^2) / (\nu_i\tau_i^2 + e_i^2)^3$, where $e_i = t_{k, i} - (W_k\chi^*)_i$ is the steady-state residual. This reveals that the skewness vanishes as either $\nu_i\tau_i^2 \gg  e_i^2$ or $\nu_i\tau_i^2 \ll  e_i^2$ (Practical applications are dominated by these two regimes). Under the symmetric assumption, by applying Lemma \ref{equivalent}, we convert \eqref{elbo_marginalized} to the following marginalized MAP (MMAP) optimization problem:
	\begin{equation}
	\begin{aligned}
		\hat{\chi}_k &= \arg\max_{\chi_k} \log p(\chi_k, t_k \mid t_{1:k-1})\\
		&= \arg\min_{\chi_k} \sum_{i=1}^{n+p+m} \underbrace{\frac{\nu_i + 1}{2} \log \left( 1 + \frac{e_i^2}{\nu_i\tau_i^2} \right)}_{J_i},
	\end{aligned}
	\label{mapmargalized}
	\end{equation}
	where $e_i= t_{k,i} -(W_k\chi_k)_i$. Note that the posterior covariance for \eqref{elbo_marginalized} does not have a closed-form solution. In addition, the exact solution for \eqref{elbo_marginalized} is slightly \emph{over-confident} due to the zero-force effect caused by $\mathrm{KL}(q||p)$, extensive discussion on $\mathrm{KL}(q||p)$ versus $\mathrm{KL}(p||q)$ is available on \cite{pml2Book}. To simplify the solution, we apply Laplace approximation as  $q(\chi_k) \approx \mathcal{N}(\hat{\chi}_k, \hat{P}_{\chi_k})$ around the MAP point, where $\hat{P}_{\chi_k})= \left( W_k^T \Lambda_{\text{Lap}} W_k \right)^{-1}$ with $\Lambda_{\text{Lap}, ii} = \left. \frac{\partial^2}{\partial e_i^2} \left[ \frac{\nu_i + 1}{2} \log \left(J_i \right) \right] \right|_{e_i = \hat{e}_i}$, proving a sufficiently good starting point. Subsequently, we advocate for updating $q(\lambda^{wv})$ via \eqref{vbsub2} in a single step. The entire algorithm, called the hierarchical variational estimator (HVE), is summarized in Algorithm \ref{hve}. A graphical comparison between standard CAVI and the proposed HVE is shown in Fig. \ref{visualize}. As demonstrated by the 1D example, the proposed method achieves similar performance with just a single $q(\lambda)$ update compared to multiple iterations of conventional CAVI.
	\begin{algorithm}[h]
		\setstretch{0.8} 
		\caption{HVE}
		\label{hve}
		\small
		\begin{algorithmic}[1]
			\State \textbf{Step 1: MMAP}\\
			Solve \eqref{mapmargalized} accompanied with Laplace approximation to obtain $\mathcal{N}(\hat{\chi}_k, \hat{P}_{\chi_k})$.
			\State \textbf{Step 2: Update $q(\lambda^{wv})$}\\
			Update $q(\lambda^{wv})$ based on \eqref{vbsub2}. 
		\end{algorithmic}
	\end{algorithm}
	\begin{remark}
		Geometrically, standard CAVI suffers from slow, zig-zag convergence trajectories due to the tightly coupled, skewed contour of the joint posterior. Our strategy circumvents this by utilizing optimization (MAP) to directly navigate to the mode (or equivalently, the mean of the marginalized posterior distribution), employing the Laplace method to capture the local curvature, and executing a single variational step to align the marginalized distributions. It is worth mentioning that after updating $q(\lambda^{wv})$, one optional action is to update $q(\chi_k)$ based on \eqref{vbsub1}. However, this procedure usually can be neglected since the MMAP solution usually provides sufficiently good approximations as demonstrated in Fig. \ref{visualize}. 
	\end{remark}
	
	\begin{figure}[htbp]
		\centering 	
		\subfigure[Contour path]{
			\begin{minipage}[t]{0.47\linewidth}
				\centering
				\includegraphics[width=1.0\columnwidth]{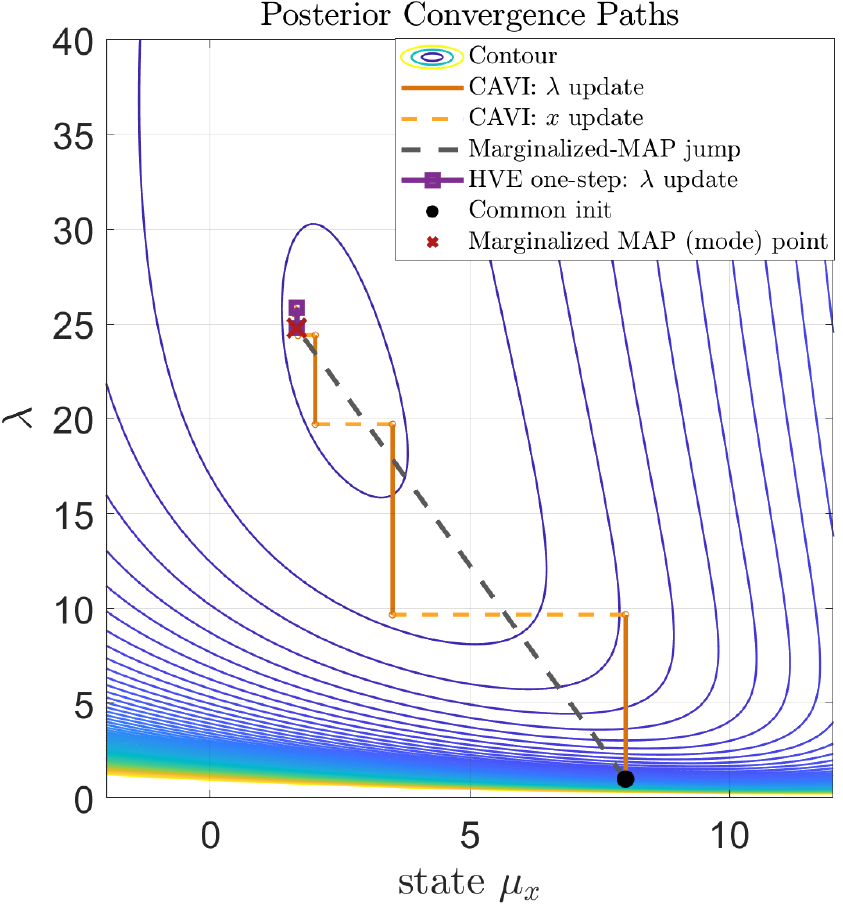}
				\label{nkf_d}
			\end{minipage}%
		}%
		\subfigure[PDF comparison]{
			\begin{minipage}[t]{0.50\linewidth}
				\centering
				\includegraphics[width=1.0\columnwidth]{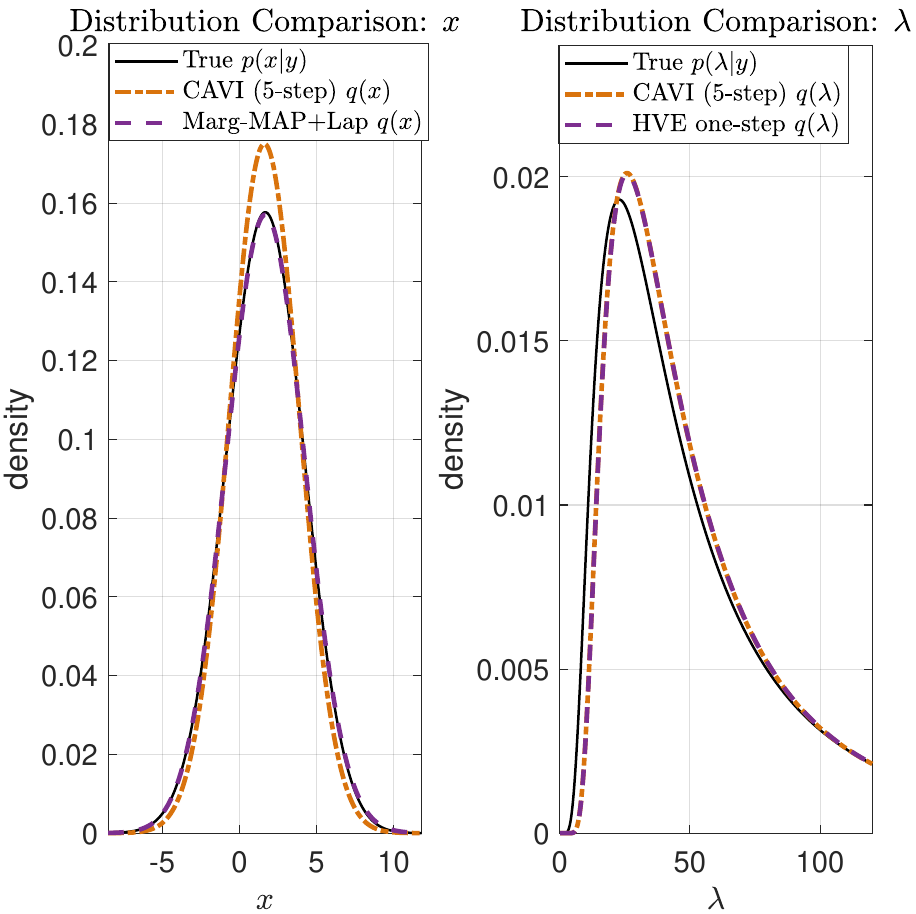}
				\label{akf_d}
			\end{minipage}%
		}%
		\caption{Comparison of CAVI and marginalized MAP+1-step inference methods in a 1D example. The model assumes a Gaussian likelihood $y \sim \mathcal{N}(x, \lambda)$ for a single measurement $y = 12.0$, with priors $x \sim \mathcal{N}(0, 6.0)$ and $\lambda \sim \text{IG}(1.5, 21.0)$. The CAVI is executed based on \eqref{vbfull} with 5 iterations. The proposed method is executed based on Algorithm \ref{hve}. (a) Visualization of the convergence paths of different methods with the same initial point. (b) The resulting approximate distributions. One observes that CAVI is slightly overconfident regarding the $q(x)$ due to the zero-forcing effect of $\mathrm{KL}(q||p)$, which is avoided by the MMAP.}
		\label{visualize}
	\end{figure}
	
	\subsection{Algorithm Derivation}
	This subsection provides details for Algorithm \ref{hve} under the whitened regression problem \eqref{wreg}. Based on \eqref{mapmargalized}, we define the following objective function:
	\begin{equation}
		\mathcal{J}(\chi_k) = \sum_{i=1}^{n+p+m} \frac{\nu_i + 1}{2} \log \left( 1 + \frac{e_i^2}{\nu_i\tau_i^2} \right).	
		\label{rbobj}
	\end{equation}	
	Recalling Proposition \ref{prop_hq}, the robust objective \eqref{rbobj} can be transformed into a sequence of tractable surrogate problems by introducing an auxiliary optimization variable $\lambda_i > 0$ for each residual term $e_i$. Thus, we have:
	\begin{equation}
		\mathcal{J}(\chi_k) = \inf_{\lambda_i > 0} \sum_{i=1}^{n+p+m}\left[ \frac{e_i^2}{2\lambda_i} + \psi(\lambda_i) \right],
	\end{equation}
	where the optimum of the augmented dual function is achieved at:
	\begin{equation}
		\lambda_i = \frac{e_i}{\rho'(e_i)} = \frac{\nu_i\tau_i^2 + e_i^2}{\nu_i + 1}.
	\end{equation}
	Consequently, minimizing the original objective function $\mathcal{J}(\chi_k)$ is mathematically equivalent to solving an iteratively reweighted least squares (IRLS) problem. Let $j$ denote the iteration index. At each iteration, we define the equivalent scalar weight as the inverse of the auxiliary variable, $w_i^{(j)} \triangleq (\lambda_i^{(j)})^{-1}$, which is computed based on the current state estimate $\chi_k^{(j)}$:
	\begin{equation}
		w_i^{(j)} = \frac{\nu_i + 1}{\nu_i\tau_i^2 + \left(e_i^{(j)}\right)^2}.
		\label{weight}
	\end{equation}
	By assembling these scalar weights into a diagonal weighting matrix $\Lambda_k(e_k^{(j)}) \triangleq \operatorname{diag}\left[w_1^{(j)}, \ldots, w_{n+p+m}^{(j)}\right]$, the non-linear optimization problem is solved by the following normal equations:
	\begin{equation}
		\chi_k^{(j+1)} = \left(W_k^{T}\Lambda_k(e_k^{(j)})W_k\right)^{-1} W_k^{T}\Lambda_k(e_k^{(j)})\,t_k,
		\label{solu1}
	\end{equation}
	until the sequence $\{\chi_k^{(j)}\}$ converges to the optimal estimate.
	
	\begin{remark}[Structural Equivalence]
		By comparing the MMAP solution in \eqref{solu1} and the CAVI in \eqref{wei}, we notice they share a similar least squares structure, differing in the handling of the weighting matrix. The two methods become equivalent when the posterior variance of the residual is negligible, such that $\mathbb{E}_{q(\chi_k)}[e_i^2] = {e}_i^2$. In applications, the residuals are generally significantly larger than the a posteriori error covariance, i.e., $\mathbb{E}_{q(\chi_k)}[e_i^2] \approxeq {e}_i^2$. This demonstrates that the MMAP solution provides a good approximation for the state mean of the KL divergence optimization. 
	\end{remark}
	\begin{remark}[Computational Efficiency]
		The MMAP offers computational efficiency compared to CAVI, as shown in \eqref{CAVIalg}. In CAVI, the posterior error covariance $P_{\chi_k}$ must be recomputed at each iteration; however, this step is bypassed in the proposed framework. Furthermore, our sequential pipeline, in which $\chi_k$ is obtained prior to updating the prior distribution parameters, provides the flexibility to seamlessly integrate robust and adaptive estimation techniques, as demonstrated in \cite{li2025variational}.
	\end{remark}
	\begin{theorem}[Monotonic Convergence]
		\label{convRWLS}
		Consider the objective function $\mathcal{J}(\chi_k)$ defined in \eqref{rbobj} and the fixed-point iteration given by \eqref{solu1}, denoted as the map $\chi_k^{(j+1)} = \mathcal{T}(\chi_k^{(j)})$. The sequence $\{\mathcal{J}(\chi_k^{(j)})\}_{j \ge 0}$ is monotonically non-increasing and converges to a stationary point of $\mathcal{J}(\chi_k)$.
	\end{theorem}
	The proof is available in Appendix \ref{fixproof}.	
	\begin{remark}
		The convergence guarantee established in Theorem \ref{convRWLS} can be extended to a broader class of robust M-estimators characterized by the objective $\mathcal{J}(\chi_k) = \sum_{i} \rho(e_i(\chi_k))$, provided the loss function $\rho(\cdot)$ satisfies the following half-quadratic conditions:
		\begin{enumerate}
			\item \textbf{Symmetry:} $\rho(e_i) = \rho(-e_i)$ for all $x \in \mathbb{R}$.
			\item \textbf{Regularity:} $\rho(e_i)$ is continuously differentiable.
			\item \textbf{Concavity:} The function $\phi(u) = \rho(\sqrt{u})$ is concave on the domain $u \in [0, \infty)$.
		\end{enumerate}
		The proof is similar to that of Theorem \ref{convRWLS} and is therefore omitted. This class encompasses several commonly used robust loss functions, including the Geman-McClure~\cite{barron2019general}, Welsch (also referred to as multi-kernel correntropy loss~\cite{a9}), smoothed Huber loss, and fair loss; see details in Appendix \ref{fixproof}. 
	\end{remark}

	By applying matrix inversion lemma [1], \eqref{solu1} can be identically expressed in a Kalman update format. To proceed, we denote $\bar{C}=[0, C]$ and ${\chi}_{k,0}^{+}={\chi}_{k}^{-}$ as the initial state estimation at $t=0$ iteration. Subsequently, for $t \ge 1$, we have (see details in \cite{a9,li2021multi})
	\begin{equation}
		\begin{aligned}
			\chi_{k,t}^{+}&=\chi_k^{-}+\tilde{K}_{k,t}({y}_k-\bar{C}\chi_k^{-})\\
			\tilde{K}_{k,t}&=\tilde{P}_{\chi k}^{-}\bar{C}^{T}(\bar{C}\tilde{P}_{\chi k}^{-}\bar{C}^{T}+\tilde{R}_{k})^{-1}\\
			\tilde{P}_{\chi k}^{-} &= \begin{bmatrix}
				\Sigma_{Pk}&\Sigma_{Pk}\\
				\Sigma_{Pk}&\Sigma_{Pk}+\tilde{Q}_{ k}
			\end{bmatrix}\\
			\tilde{Q}_{ k}& = B \Lambda_w^{-1} B^T\\
			\Lambda_{w}&=\operatorname{diag}[d_{\nu_{n+1}}(e_{n+1}),\ldots,d_{\nu_{n+p}}(e_{n+p})]\\
			\tilde{R}_{k}&=D{\Lambda}_{v}^{-1}D^{T}\\	\Lambda_{v}&=\operatorname{diag}[d_{\nu_{n+p+1}}(e_{n+p+1}),\ldots,d_{\nu_{n+p+m}}(e_{n+p+m})]\\
			e_k &= t_k - W_k \chi_k^{t-1}
		\end{aligned}
		\label{update}
	\end{equation} 
	where the cross term of $\tilde{P}_{\chi k}^{-}$ comes from the fact that $E[(m_k-m_k^{-})(x_k-x_k^{-})^{T}]=E[(-p_k)(-p_k+w_k)^{T}]= \Sigma_{P_k}$. This process is iterated until convergence, subject to a predefined termination condition or a maximum number of iterations. At the final iteration, one can update the posteriori error covariance as:
	\begin{equation}
		P_{\chi k}^{+} = (I - \tilde{K}_{k,t} \bar{C}) {P}_{\chi k}^{-} (I - \tilde{K}_{k,t} \bar{C})^T + \tilde{K}_{k,t} {R}_{k} \tilde{K}_{k,t}^{T}.
		\label{errcov}
	\end{equation}
	The complete algorithm, named as MMAP estimator, is summarized in Algorithm \ref{Alg1}.
	An alternative expression of \eqref{update} and \eqref{errcov}, which update $m_k$ and $x_k$ explicitly,  is also provided in Appendix \ref{alternate}.
	\begin{algorithm}[bt]
		\setstretch{0.8} 
		\caption{MMAP}
		\label{Alg1}
		\small
		\begin{algorithmic}[1]
			\State \textbf{Step 1: Initialization}\\
			Choose $\nu_i$, $\tau_i^2$ for channel $i$, and select the maximum iteration $m_{iter}$ and a threshold $\varepsilon$.
			\State \textbf{Step 2: State Prediction}\\
			${x}_{k}^{-}=A \hat{{x}}_{k-1}^{+}$, ${m}_{k}^{-}=A \hat{{x}}_{k-1}^{+}$\\
			Obtain ${B}_{k}$ according to \eqref{Bk}\\
			${\chi}_{k,0}^{+}={\chi}_{k}^{-}$\\
			Obtain $t_k$ and $W_k$ through \eqref{tkwk}
			\State \textbf{Step 3: State Update}
			\While{$\frac{\left\|\chi_{k,t}^{+}-\chi_{k,t-1}^{+}\right\|}{\left\|\chi_{k,t}^{+}\right\|}>\varepsilon$ and $t \le m_{iter}$}
			\State Execute \eqref{update}
			\State $t=t+1$
			\EndWhile
			\State Update ${P}_{k}^{+}$ using \eqref{errcov}
			
			\State \textbf{Step 4: Recursion}
			\State Set $k \leftarrow k+1$ and return to Step 2.
		\end{algorithmic}
	\end{algorithm}
	
	Let $\mathcal{I} = \{n+1, \dots, n+p+m\}$ denote the set of \emph{inference} channels selected for covariance inference, and $\mathcal{F} = \{1, \dots, n\}$ denotes the frozen set. We then focus on the update of $\{\lambda_i\}_{i \in \mathcal{I}}$. By applying mean-field approximation, we have 
	\begin{equation}
		q(\chi_k, \{\lambda_i\}_{i \in \mathcal{I}}) = q(\chi_k) \prod_{i \in \mathcal{I}} q(\lambda_i).
		\label{mean_field_assumption}
	\end{equation}
	According to Appendix \ref{CAVIalg}, the one-step $q(\lambda_i)$ update follows
	\begin{equation}
		q(\lambda_i) = \operatorname{IG}\left(\lambda_i \middle| \frac{\nu_i^*}{2}, \frac{\nu_i^*(\tau_i^{2})^{*}}{2}\right),
		\label{posterior_lambda}
	\end{equation}
	where
	\begin{equation}
		\begin{aligned}
			\nu_i^* &= \nu_i + 1, \\
			(\tau_i^{2})^{*} &= \frac{\mathbb{E}_{q(\chi_k)}[([t_k - W_k \chi_k]_i)^2]}{\nu_i+1} +  \frac{\nu_i \tau_i^2}{\nu_i+1},
		\end{aligned}
		\label{posterior_params}
	\end{equation}
   and
	\begin{equation}\notag
		\mathbb{E}_{q(\chi_k)}[([t_k - W_k \chi_k]_i)^2] = \left[t_k - W_k {\chi}_k^{+}\right]_i^2 + \left[W_k P_{\chi k}^{+} W_k^T\right]_{ii}.
	\end{equation}

	\begin{remark}
		By lifting the state space to explicitly incorporate the process noise, \eqref{posterior_params} structurally unifies the process and measurement models. This enables the simultaneous variational inference of both process and measurement noise covariances within a single, unified update framework.
	\end{remark}
	
	\begin{prop} (Scale Invariance).
		\label{scaleinv}
		For the system defined in \eqref{sys}, if the initial error covariance $P_0$ and the noise covariances $Q_k$ and $R_k$ are scaled by a constant $\alpha > 0$, then for all $k \ge 1$, the Kalman gain $K_k$ and state estimates $\hat{x}_{k|k-1}$ and $\hat{x}_{k|k}$ remain invariant. Meanwhile, the error covariances are scaled by $\alpha$, i.e., $P'_{k|k-1} = \alpha P_{k|k-1}$ and $P'_{k|k} = \alpha P_{k|k}$.
	\end{prop}
	The proof is available in Appendix \ref{pscale}. Proposition \ref{scaleinv} implies that simultaneously estimating all channels of $Q_k$ and $R_k$ (equivalently, $\{\lambda_i\}_{i \in \mathcal{I}}$) is ill-posed. Furthermore, in practical applications, It is common that only a subset of the process and measurement noise statistic characteristics change over time. To reflect this fact, we split \eqref{sys} as:
	\begin{equation}
		\begin{aligned}
			x_{k}&=A x_{k-1} + B_f w_{fk} + B_v w_{vk} \\
			y_k &= C x_{k} +D_f v_{fk} + D_v v_{vk}
			\label{syssplit}
		\end{aligned}
	\end{equation}
	where $B_f w_{fk}$ and $D_f v_{fk}$ denote the fixed components, while $B_v w_{vk}$ and $D_v v_{vk}$ represent the time-varying components to be inferred. We define an active set $\mathcal{A} \in \mathcal{I}$ which is a subset of $\mathcal{I}$ and contains the non-zero channels of $B_v w_{vk}$ and $D_v v_{vk}$. Consequently, one can perform inference on $\mathcal{A}$ to avoid ill-posedness.
	
	\begin{remark}
		Consider the partitioned system in \eqref{syssplit}. If the fixed noise components are positive-definite (i.e., $B_f Q_f B_f^T \succ 0$ or $D_f R_f D_f^T \succ 0$), the uniform scale invariance established in Proposition \ref{scaleinv} is strictly broken. 
	\end{remark}
	The update rule $\nu_i^* = \nu_i + 1$ in \eqref{posterior_params} implies that the degrees of freedom increase monotonically over time. As $\nu_i \to \infty$, the innovation weight $\frac{1}{\nu_i+1}$ vanishes, causing the covariance inference to saturate and lose sensitivity to new data. Meanwhile, as shown by Property \ref{pro2}, the increment of $\nu_i$ suppresses the robust mechanism by forcing the weight function $d_{\nu_i}(e_i)$ to converge to the constant $1/\tau_i^2$, degenerating the robust estimator into a standard  Kalman filter. To maintain persistent adaptivity, we introduce a forgetting factor $0 < \rho_i < 1$ to bound the growth of the degrees of freedom. The sequential update for $\nu_i$ is modified as:
	\begin{equation}
		\begin{aligned}
			\nu_{i,k|k-1} &= \rho_i \nu_{i,k-1|k-1} \\
			\nu_{i,k|k} &= \nu_{i,k|k-1} + 1.
		\end{aligned}
	\end{equation}
	The whole hierarchical variational Kalman filtering (HVKF) is summarized in Algorithm \ref{HVF}.
	\begin{algorithm}[h]
		\setstretch{0.8} 
		\caption{HVKF}
		\label{HVF}
		\small
		\begin{algorithmic}[1]
			\State \textbf{Step 1: Initialization} \\
			Set initial $\nu_{i,0|0}$, $\tau^2_{i,0|0}$, and forgetting factors $\rho_i \in (0, 1)$.
			\State \textbf{Step 2: Hyperparameter Prediction}
			\For{$i \in \mathcal{A}$}
			\State $\nu_{i,k|k-1} = \rho_i \nu_{i,k-1|k-1}$
			\State $\tau^2_{i,k|k-1} = \tau^2_{i,k-1|k-1}$ \Comment{Static model}
			\EndFor
			
			\State \textbf{Step 3: State Estimation (MMAP)}
			\State Input $\nu_{i,k|k-1}$ and $\tau^2_{i,k|k-1}$ into \textbf{Algorithm \ref{Alg1}}
			\State Execute \textbf{Algorithm \ref{Alg1}} to obtain ${\chi}_k^{+}$ and $P_{\chi k}^+$.
			\State \textbf{Step 4: Hyperparameter Posterior Update}
			\For{$i \in \mathcal{A}$} 
			\State $\nu_{i,k|k} = \nu_{i,k|k-1} + 1$
			\State $(\tau_i^2)_{k|k} =\frac{\nu_{i,k|k-1} \tau^2_{i,k|k-1}}{\nu_{i,k|k}} + \frac{\left[t_k - W_k {\chi}_k^{+}\right]_i^2 + \left[W_k P_{\chi k}^{+} W_k^T\right]_{ii}}{\nu_{i,k|k}}$
			\EndFor
			
			\State \textbf{Step 5: Recursion}
			\State Set $k \leftarrow k+1$ and return to Step 2.
		\end{algorithmic}
	\end{algorithm}
	\begin{remark}
	It is worth mentioning that $\nu_{i,k|k}$ controls how confident we are about our current estimate, and a smaller value indicates a larger uncertainty. In the steady state, one has $\nu_{i,k|k}=\frac{1}{1-\rho_i}$, demonstrating that the steady state $\nu_{i,k|k}$ indeed is controlled by the forgetting factor.
	\end{remark}
	\subsection{Covariance Tracking Analysis and Delay Compensation}
	The whitened and lifted modeling in \eqref{tkwk} allows us to analyze the covariance tracking speed channel by channel. At the steady state, one has $\rho_i = \frac{\nu_{i,k|k-1}}{\nu_{i,k|k-1}+1}$. Defining the realized variational error as $\mathcal{E}_{i,k} \triangleq [t_k - W_k {\chi}_k^{+}]_i^2 + [W_k P_{\chi k}^{+} W_k^T]_{ii}$, we assume it is a weakly stationary sequence with mean $\sigma_i^2$ and variance $\eta_i^4$. Subsequently, we rewrite the scale parameter update function as shown in Line 14 of Algorithm \ref{HVF} (at the steady state) as 
	\begin{equation}
		(\tau_{i,k|k})^2 = \rho_i (\tau_{i,k-1|k-1})^{2} + (1-\rho_i) \mathcal{E}_{i,k}.
		\label{scale_update}
	\end{equation} 
	\begin{theorem} \label{prop:refined_estimator}
		Defining the prior steady-state degrees of freedom $\nu_i^{-} \triangleq \nu_{i,k|k-1}$. The scale parameter update in \eqref{scale_update} forms an exponentially weighted moving average (EWMA) estimator, 
		which satisfies the following steady-state properties:
		\begin{enumerate}
			\item \emph{Asymptotic Unbiasedness:} $\lim_{k \to \infty} \mathbb{E}\left[ (\tau_i^2)_{k|k} \right] = \sigma_i^2$.
			\item \emph{Tracking Variance:} $\lim_{k \to \infty} \text{Var}\left( (\tau_i^2)_{k|k} \right) = \frac{1-\rho_i}{1+\rho_p}\eta_i^4= \frac{1}{2\nu_i^{-} + 1} \eta_i^4$.
			\item \emph{Tracking Time Constant:} $T_i =\frac{-1}{\ln(\rho_i)}$.
		\end{enumerate}
	\end{theorem}
	The proof is available in Appendix \ref{prefined_estimator}. It is obvious that there is a trade-off between the covariance tracking speed and tracking smoothness based on Theorem \ref{prop:refined_estimator}, i.e., a bigger forgetting factor $\rho_i$ gives a smaller tracking variance but a bigger tracking time constant. To address this, we provide two remedies: a higher-order filter and a sliding window filter, which will be detailed below.
	\subsubsection{High Order Filter} 
	Here we provide a general prior model. For the scale parameter update, we design  
	\begin{equation}
		\tau_{i,k|k-1}^2 = \sum_{m=1}^{N} \phi_m \tau_{i,k-m|k-m}^2,
		\label{momentump}
	\end{equation}
	with $\sum_{m=1}^{N} \phi_m = 1$. Substituting it into \eqref{scale_update} gives
	\begin{equation}
		\tau_{i,k|k}^2 = \rho_i \sum_{m=1}^{N} \phi_m \tau_{i,k-m|k-m}^2 + (1-\rho_i) \mathcal{E}_{i,k}.
		\label{scale_update1}
	\end{equation} 
	\begin{prop}
		\label{BIBOstable}
		The estimator \eqref{scale_update1} is BIBO stable if and only if all roots of its corresponding characteristic polynomial, $P(z) = z^N - \sum_{m=1}^{N} \rho_i \phi_m z^{N-m} = 0$, lie strictly inside the unit circle.
	\end{prop}
	The proof can be obtained by taking the Z-transform of \eqref{scale_update1} and applying standard linear system stability criteria~\cite{ogata1995discrete}. We omitted it here for brevity.  In the case $N=1$, one has $\phi_1=1$, which recovers the EWMA estimator. In the case $N=2$, we obtain a velocity compensation prior by assigning $\phi_1= 1+ \gamma$ and $\phi_2=-\gamma$ so that $$
	\tau_{i,k|k-1}^2 =  \tau_{i,k-1|k-1}^2  +   \gamma (\tau_{i,k-1|k-1}^2- \tau_{i,k-2|k-2}^2).
	$$ 
	This explicitly calculates the first-order difference (trend) of the variance. Based on Proposition \ref{BIBOstable}, one has $P(z) = z^2 - \rho_i(1+\gamma)z + \rho_i \gamma$. The system is strictly BIBO stable if the following condition holds:
	\begin{equation}
		0 \le \gamma < \frac{1}{\rho_i}.
	\end{equation}
	In the case $N=3$, we obtain an acceleration compensation prior. Denote $\Delta \tau_{i,k-1}^2 \triangleq \tau_{i,k-1|k-1}^2- \tau_{i,k-2|k-2}^2$. We formulate the prior as 
	\begin{equation}
		\tiny
		\tau_{i,k|k-1}^2 = \tau_{i,k-1|k-1}^2 + \gamma_1 (\Delta \tau_{i,k-1}^2) + \gamma_2 \left[ \Delta \tau^2_{i,k-1} - \Delta \tau^2_{i,k-2} \right].
	\end{equation}
	By expanding the differences, the corresponding filter weights are exactly synthesized as $\phi_1 = 1+\gamma_1+\gamma_2$, $\phi_2 = -(\gamma_1+2\gamma_2)$, and $\phi_3 = \gamma_2$. This 3rd-order design dramatically reduces dynamic tracking lag during parabolic or exponentially growing profiles by predicting the rate-of-change of the variance trend itself. However, the stability margin becomes tighter. The roots of $P(z) = z^3 - \rho_i\phi_1 z^2 - \rho_i\phi_2 z - \rho_i\phi_3$ must be verified via the 3rd-order Jury test, requiring $\gamma_2 < \gamma_1 < 1/\rho_i$ to prevent underdamped oscillatory divergence.
	
	We visualize the trade-off between tracking speed and steady-state variance for the $N=1$ and $N=2$ estimators in Fig. \ref{Pareto}. For the $N=2$ case, setting $\gamma = \frac{\rho^2}{1+\rho}$ yields two real poles: $z_1 = \rho^2$ and $z_2 = \frac{\rho}{1+\rho}$. The $N=2$ time constant is dictated by the dominant pole ($z_1 = \rho^2$ as $\rho \ge 0.618$), while the steady-state variance is derived based on Theorem \ref{prop:refined_estimator}. The results confirm that the $N=2$ scenario possesses a better Pareto profile compared with $N=1$. The detailed high-order HVKF (HOHVKF) is summarized in Algorithm \ref{HOHVKF}.
	\begin{figure}[!htp]
		\centering
		\includegraphics[width=0.98\columnwidth]{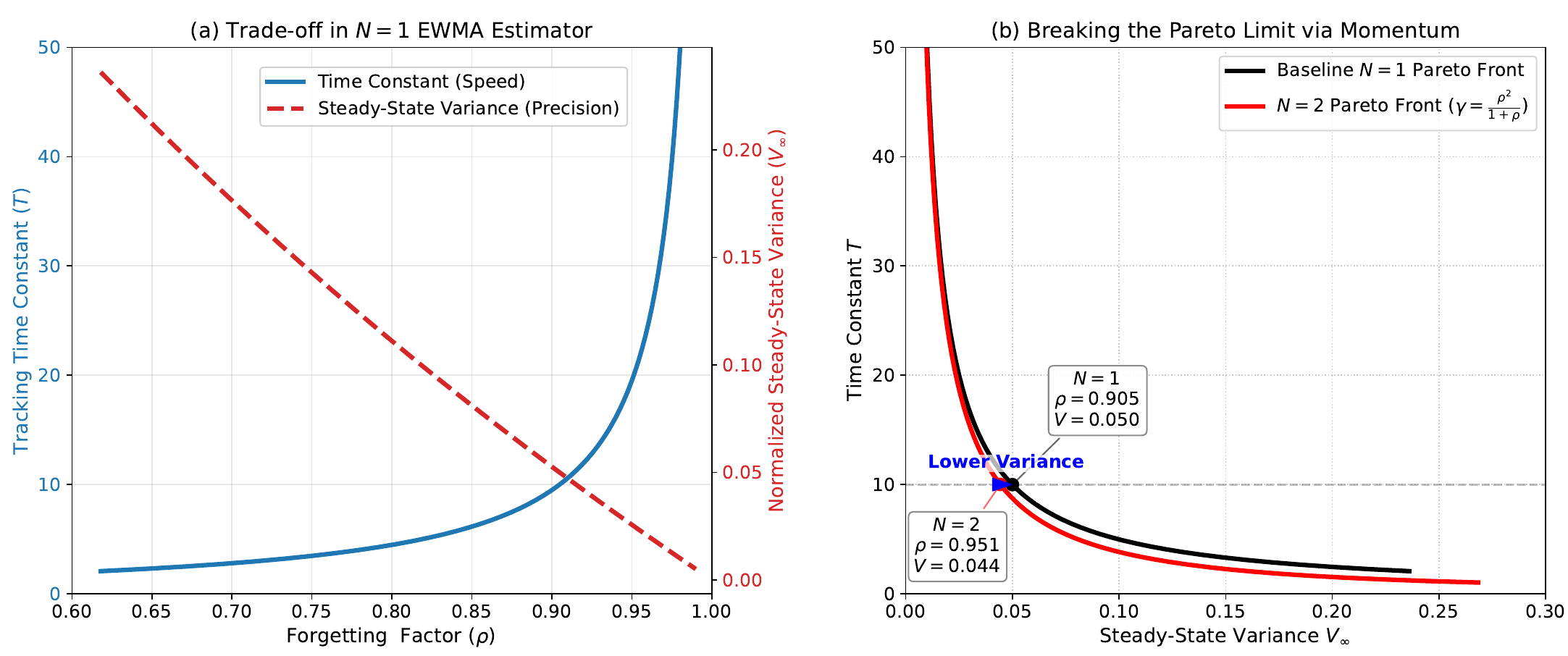}
		\caption{Performance analysis of the $N=1$ and $N=2$ estimators. (a) The inherent speed-precision trade-off in the standard $N=1$ EWMA filter, plotting the time constant (solid blue line) and normalized steady-state variance (dashed red line) against $\rho$. (b) Pareto front comparison of tracking speed versus precision. By applying the optimal overdamped momentum gain $\gamma = \frac{\rho^2}{1+\rho}$, the $N=2$ trajectory (solid red line) successfully breaks the absolute Pareto limit of the $N=1$ baseline (solid black line), achieving strictly lower variance at equivalent tracking speeds.}
		\label{Pareto}
	\end{figure}
	
	\begin{algorithm}[h]
		\setstretch{0.8} 
		\caption{HOHVKF}
		\label{HOHVKF}
		\small
		\begin{algorithmic}[1]
			\State \textbf{Step 1: Initialization} \\
			Set initial $\nu_{i,0|0}$, $\tau^2_{i,0|0}$, $\rho_i \in (0, 1)$ for $i \in \mathcal{A}$. Additionally, set the filter parameter $\{\phi_m\}_{m=1:N}$.
			
			\State \textbf{Step 2: Hyper-parameter Prediction}
			\For{$i \in \mathcal{A}$}
			\State $\nu_{i,k|k-1} = \rho_i \nu_{i,k-1|k-1}$
			\State $\tau^2_{i,k|k-1} = \sum_{m=1}^{N} \phi_m \tau_{i,k-m|k-m}^2$ \Comment{High-order model}
			\EndFor
			
			\State \textbf{Execute Step 3, 4, and 5 in Algorithm \ref{HVF}}
		\end{algorithmic}
	\end{algorithm}
	
	\subsubsection{Sliding Window Filter}
	By investigating \eqref{scale_update}, we find that $\mathcal{E}_{i,k}$ is influenced by the estimate of $\chi_k$, which is coupled with the preceding scale parameter $\tau_{i,k-1|k-1}^2$. This temporal dependency indicates that a refined historical estimate of $\tau_{i,k-1|k-1}^2$ contributes to a tighter current scale estimate. Motivated by this, we formulate a sliding-window scale parameter estimator. Let $L$ denote the length of the sliding window. At time step $k$, we define the sequence of the localized scale parameters and the corresponding realized variational errors over the window as vectors:
	\begin{equation}\notag
		\begin{aligned}
		\boldsymbol{\mathcal{T}}_{i,k} &\triangleq \begin{bmatrix} \tau_{i,k-L+1|k-L+1}^2, & \dots, & \tau_{i,k|k}^2 \end{bmatrix}^T, \\
		\boldsymbol{\mathcal{E}}_{i,k} &\triangleq \begin{bmatrix} \mathcal{E}_{i,k-L+1}, & \dots, & \mathcal{E}_{i,k} \end{bmatrix}^T.
		\end{aligned}
	\end{equation}
	To obtain an optimal and smooth estimation of the sequence $\boldsymbol{\mathcal{T}}_{i,k}$ given the measurement sequence $\boldsymbol{\mathcal{E}}_{i,k}$, we cast the scale parameter update as a regularized localized optimization problem:
	\begin{equation} 
		\tiny
		\label{eq:sliding_window_opt}
		\hat{\boldsymbol{\mathcal{T}}}_{i,k} = \arg\min_{\boldsymbol{\mathcal{T}}_{i,k}} \left\{ \sum_{j=k-L+1}^{k} \frac{1}{2}\mathcal{D} \left( \mathcal{E}_{i,j}, \tau_{i,j|j}^2 \right) + \frac{\lambda}{2} \mathcal{R}(\boldsymbol{\mathcal{T}}_{i,k}) \right\}.
	\end{equation}
	Here, $\mathcal{D}(\cdot, \cdot)$ is a data fidelity penalty (typically the squared Euclidean distance $\left\| \mathcal{E}_{i,j} - \tau_{i,j|j}^2 \right\|_2^2$), and $\lambda > 0$ is the regularization parameter controlling the trade-off between data fidelity and tracking smoothness. The penalty function $\mathcal{R}(\boldsymbol{\mathcal{T}}_{i,k})$ explicitly encodes our prior belief regarding the dynamics of the scale parameter. Depending on the physical nature of the noise characteristics, three distinct smoothing strategies can be adopted as follows.
	\begin{itemize}
		\item \textbf{$\ell_2$-Regularization:} 
		\begin{equation}\notag
			\mathcal{R}_{\ell_2}(\boldsymbol{\mathcal{T}}_{i,k}) = \sum_{j=k-L+1}^{k} \left( \tau_{i,j|j}^2 - \tau_{i,j-1|j-1}^2 \right)^2.
		\end{equation}
		\item \textbf{$\ell_1$-Regularization:}
		\begin{equation}\notag
			\mathcal{R}_{\ell_1}(\boldsymbol{\mathcal{T}}_{i,k}) = \sum_{j=k-L+1}^{k} \left| \tau_{i,j|j}^2 - \tau_{i,j-1|j-1}^2 \right|.
		\end{equation}
		\item \textbf{$\ell_0$-Regularization:} 
		\begin{equation}\notag
			\mathcal{R}_{\ell_0}(\boldsymbol{\mathcal{T}}_{i,k}) = \sum_{j=k-L+1}^{k} \left\| \tau_{i,j|j}^2 - \tau_{i,j-1|j-1}^2 \right\|_0.
		\end{equation}
	\end{itemize}
	\begin{theorem}
		\label{prop:sliding_smoothing_equivalence}
		Let $\hat{\boldsymbol{\mathcal{T}}}_{i,k} \triangleq [(\tau_{i,k-L+1|k})^2, \dots, (\tau_{i,k|k})^2]^T$ be the global minimizer of the localized $\ell_2$-regularized objective:
		\begin{equation} \label{eq:sliding_l2_objective}
			\scalebox{0.8}{$
				\mathcal{J}(\boldsymbol{\mathcal{T}}_{i,k}) = \frac{1}{2}\sum_{j=k-L+1}^{k}\left( \tau_{i,j}^2 - \mathcal{E}_{i,j} \right)^2 + \frac{\lambda}{2} \sum_{j=k-L+1}^{k} \left( \tau_{i,j}^2 - \tau_{i,j-1}^2 \right)^2.$}
		\end{equation}
		If the regularization parameter is chosen as $\lambda = \frac{\rho_i}{(1-\rho_i)^2}$, the optimal smoothed estimates $(\tau_{i,j|k})^2$ can be exactly computed via a recursive two-pass algorithm (at the steady state sense):
		\begin{enumerate}
			\item \emph{Forward pass:} 
			\begin{equation}
				(\tau_{i,j|j})^2 = \rho_i (\tau_{i,j-1|j-1})^2 + (1-\rho_i) \mathcal{E}_{i,j}.
				\label{forward}
			\end{equation}
			\item \emph{Backward pass:} 
			\begin{equation}
				(\tau_{i,j|k})^2 = (1-\rho_i) (\tau_{i,j|j})^2 + \rho_i (\tau_{i,j+1|k})^2.
				\label{backward}
			\end{equation}
		\end{enumerate}
	\end{theorem}
	The proof is available in Appendix \ref{psliding}. The whole algorithm, termed as Sliding Window Hierarchical Variational Kalman Filtering (SWHVKF), is summarized in  Algorithm \ref{Alg:sliding_window}.
	\begin{algorithm}[h]
		\setstretch{0.8} 
		\caption{SWHVKF with $\ell_2$ Regularization.}
		\label{Alg:sliding_window}
		\small
		\begin{algorithmic}[1]
			\State \textbf{Step 1: Sliding Window Initialization} \\
			Establish a sliding window of length $L$ by running MMAP in Algorithm \ref{Alg1}. 
			\State \textbf{Step 2: Forward Pass}
			\For{$i \in \mathcal{A}$}
			\For{$j = k-L+1$ \textbf{to} $k$}
			\State $(\tau_{i,j|j})^2 = \rho_i (\tau_{i,j-1|j-1})^2 + (1-\rho_i) \mathcal{E}_{i,j}$
			\EndFor
			\EndFor
			\State \textbf{Step 3: Backward Pass}
			\For{$i \in \mathcal{A}$}
			\For{$j = k-1$ \textbf{down to} $k-L+1$}
			\State $(\tau_{i,j|k})^2 = (1-\rho_i) (\tau_{i,j|j})^2 + \rho_i (\tau_{i,j+1|k})^2$
			\EndFor
			\EndFor
			
			\State \textbf{Step 4: Standard Kalman Filter Rerun}
			\State Initialize the state at $j = k-L$ and reconstruct noise covariance sequence $\boldsymbol{\mathcal{T}}_{i,k} = \{(\tau_{i,j|k})^2\}_{i \in \mathcal{A}, j=k-L+1:k}$.
			\For{$j = k-L+1$ \textbf{to} $k$}
			\State Execute Kalman filter  to obtain ${\chi}_{j|j}^{+}$ and $P_{\chi_j|j}^{+}$.
			\EndFor
			\State \textbf{Step 5: Sliding Recursion}
			\State Set $k \leftarrow k+1$, shift the sliding window, and return to Step 1.
		\end{algorithmic}
	\end{algorithm}
	\begin{remark}
		Unlike the high-order estimator in \eqref{scale_update1}, which relies on a momentum prior for causal prediction to mitigate tracking lag, SWHVKF refines the estimation by exploiting a sliding window of historical information through a forward-backward manner. In the extreme offline case where the sliding window encompasses the entire data trajectory, this two-pass operation mathematically converges to a globally optimal zero-phase smoother, completely eliminating the time-delay (phase-shift) distortion inherent to one-sided exponential tracking.
	\end{remark}

	\section{Simulation}
	This section presents numerical examples to demonstrate the effectiveness of the proposed methods.
	\subsection{Example 1: Fast Measurement Covariance Tracking}
	We formulate the target tracking problem using the following discrete-time linear state-space model:
	$$\begin{aligned}
		x_{k+1} &= A x_k + B w_k \\
		y_k &= C x_k + v_k
	\end{aligned}$$
	where the system matrices are defined as $A = \begin{bmatrix} 1 & \Delta t \\ 0 & 1 \end{bmatrix}$, $B = \begin{bmatrix} 0.5(\Delta t)^2 \\ \Delta t \end{bmatrix}$, and $C = \begin{bmatrix} 1 & 0 \end{bmatrix}$, with the sampling interval $\Delta t = 0.01$. The process noise is assumed to be $w_k \sim \mathcal{N}(0, 1)$. To simulate abrupt changes in measurement quality, we introduce a piecewise-constant measurement noise profile $v_k \sim \mathcal{N}(0, R_k)$, where 
	$$R_k = \begin{cases}
		1, & \text{if } k \in \mathbb{S}_1 \\
		100, & \text{if } k \in \mathbb{S}_2.
	\end{cases}
	$$
	Here, $\mathbb{S} = \mathbb{S}_1 \cup \mathbb{S}_2$ represents the entire sequence of time steps. In the simulations, the root mean square error (RMSE) is calculated as $\frac{1}{nT}\sum_{k=1}^{T}\sum_{i=1}^{n}e_{i,k}$, where $T$ is the trajectory length and $N$ is the state dimension.

	During the simulation, we define $\rho=[1,1,\rho_3]$ with $\rho_3=0.99$. Note that $\rho_i=1$ indicates that variational inference is not applied to that specific channel. For the HOHVKF, we set $\gamma= 0.4925$, which yields the following scale update equations:
	\begin{equation}
		\begin{aligned}
			\tau_{i,k|k-1}^2 &=  \tau_{i,k-1|k-1}^2  +   0.4925 (\tau_{i,k-1|k-1}^2- \tau_{i,k-2|k-2}^2)\\
			(\tau_{i,k|k})^2 &= \rho_i (\tau_{i,k|k-1})^{2} + (1-\rho_i) \mathcal{E}_{i,k}.
		\end{aligned}
	\end{equation}	
	For the SWHVKF, we set the window length as $L=5$. We use the same nominal parameters across all evaluated filters: KF, VBKF~\cite{a4}, MMAP, HVKF, HOHVKF, and the Ideal KF (which utilizes the ground truth covariance). The comparative results are visualized in Fig. \ref{HOHVKF_SW}, and the measurement covariance tracking trajectory is visualized in Fig. \ref{GammaHO}. 
	\begin{figure}[!htp]
		\centering
		\includegraphics[width=0.99\columnwidth]{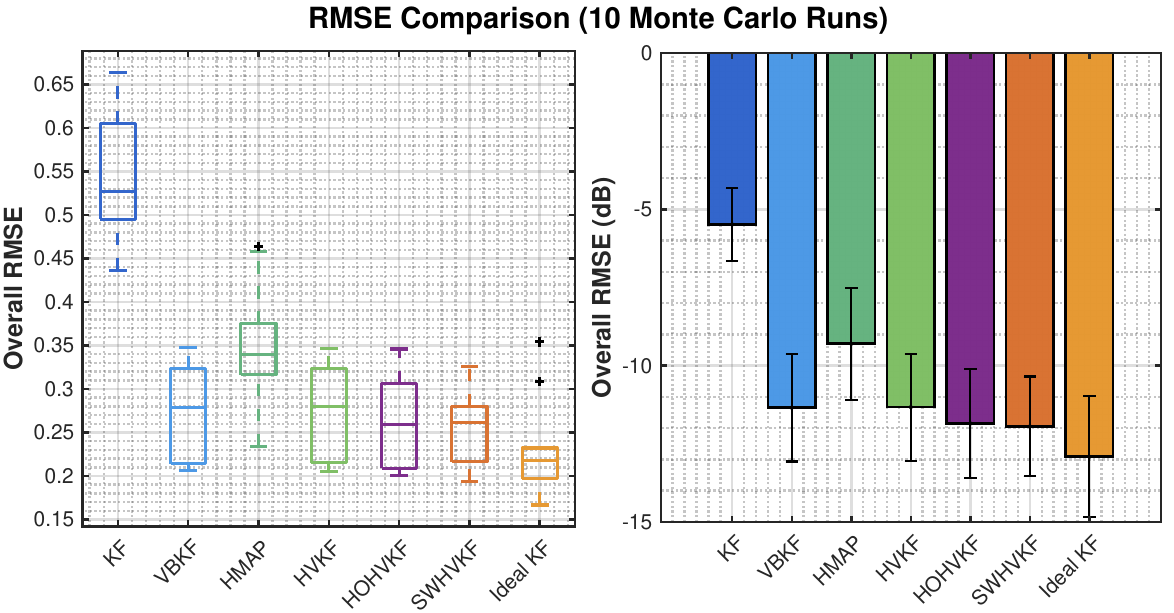}
		\caption{Overall RMSE comparison over 10 Monte Carlo runs. The left and right panels show the error distribution and mean values (dB), respectively. Compared to the standard KF, all variants achieve significant accuracy gains. Notably, both HOHVKF and SWHVKF perform better than the baseline HVKF, achieving $51.47\%$ and $52.15\%$ improvement compared with the standard KF.}
		\label{HOHVKF_SW}
	\end{figure}
	
	As shown in Fig. \ref{GammaHO}, the HOHVKF responds more than twice as fast as the HVKF, at the minor expense of slightly increased steady-state variance. This trade-off confirms the superiority of the proposed higher-order method. The detailed quantitative metrics are summarized in Table \ref{tab:overall_rmse_comparison}. The results indicate that both the proposed HOHVKF and SWHVKF perform better than the baseline HVKF, achieving accuracy levels closely approximating the ideal KF.
	
	\begin{figure}[!htp]
		\centering
		\includegraphics[width=0.75\columnwidth]{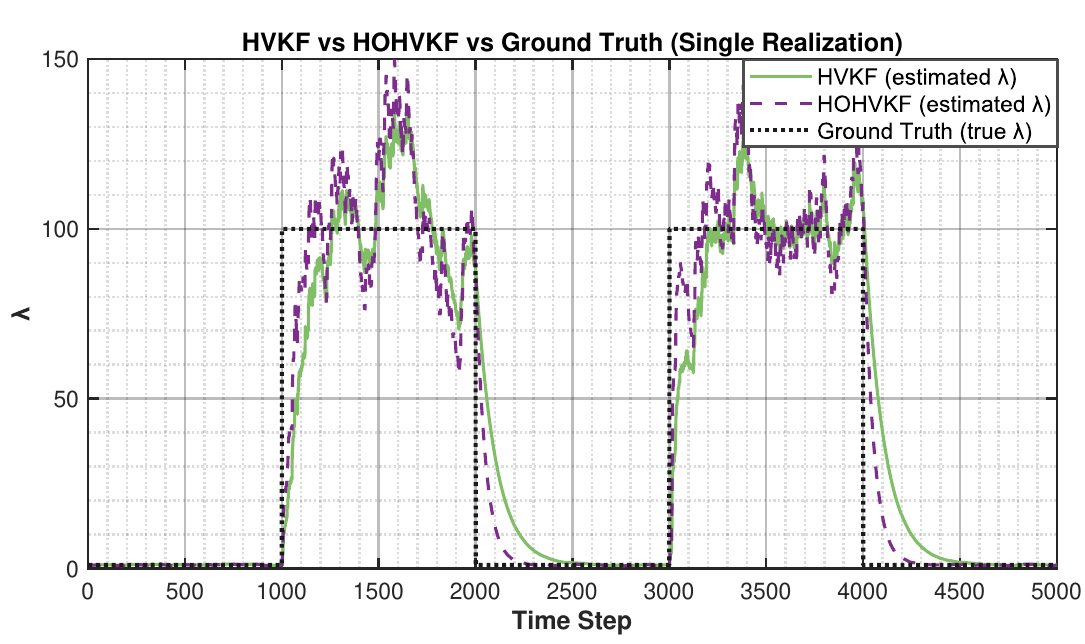}
		\caption{Tracking performance of parameter $\lambda$ in a single realization. HOHVKF responds faster to abrupt step changes in the ground truth compared to HVKF, but at the cost of slightly larger steady state variance.}
		\label{GammaHO}
	\end{figure}
	\begin{table}[ht]
		\centering
		\caption{RMSEs Comparison of Estimators.}
		\label{tab:overall_rmse_comparison}
		\scalebox{0.80}{
			\begin{tabular}{lccc}
				\hline
				\hline
				{Estimator} & { RMSE} & { RMSE (dB)} & {Improvement (\%)} \\
				\hline
				KF       & $0.536 \pm 0.072$ & $-5.48 \pm 1.16$  & -- \\
				VBKF     & $0.276 \pm 0.053$ & $-11.34 \pm 1.72$ & $+48.60$ \\
				MMAP     & $0.349 \pm 0.071$ & $-9.30 \pm 1.79$  & $+34.85$ \\
				{HVKF}   & {0.276 $\pm$ 0.053} & {-11.33 $\pm$ 1.71} & {+48.51} \\
				{HOHVKF} & {0.260 $\pm$ 0.052} & {-11.85 $\pm$ 1.74} & {+51.47} \\
				{SWHVKF} & {0.257 $\pm$ 0.047} & {-11.95 $\pm$ 1.60} & {+52.15} \\
				Ideal KF & $0.232 \pm 0.057$ & $-12.91 \pm 1.95$ & $+56.80$ \\
				\hline
				\hline
		\end{tabular}}
	\end{table}
	
	We further evaluate the performance of the proposed method over the entire measurement sequence under $\ell_2$, $\ell_1$, and $\ell_0$ regularizations (solved via ADMM with hard thresholding). The results, depicted on the left side of Fig. \ref{REGHVKF}, verify the equivalence between the $\ell_2$ regularization optimization and the two-pass filter. Furthermore, we observe that $\ell_1$ regularization achieves superior performance compared to $\ell_2$ regularization, albeit at the cost of increased computational complexity. 
	\begin{figure}[!h]
		\centering
		\includegraphics[width=0.99\columnwidth]{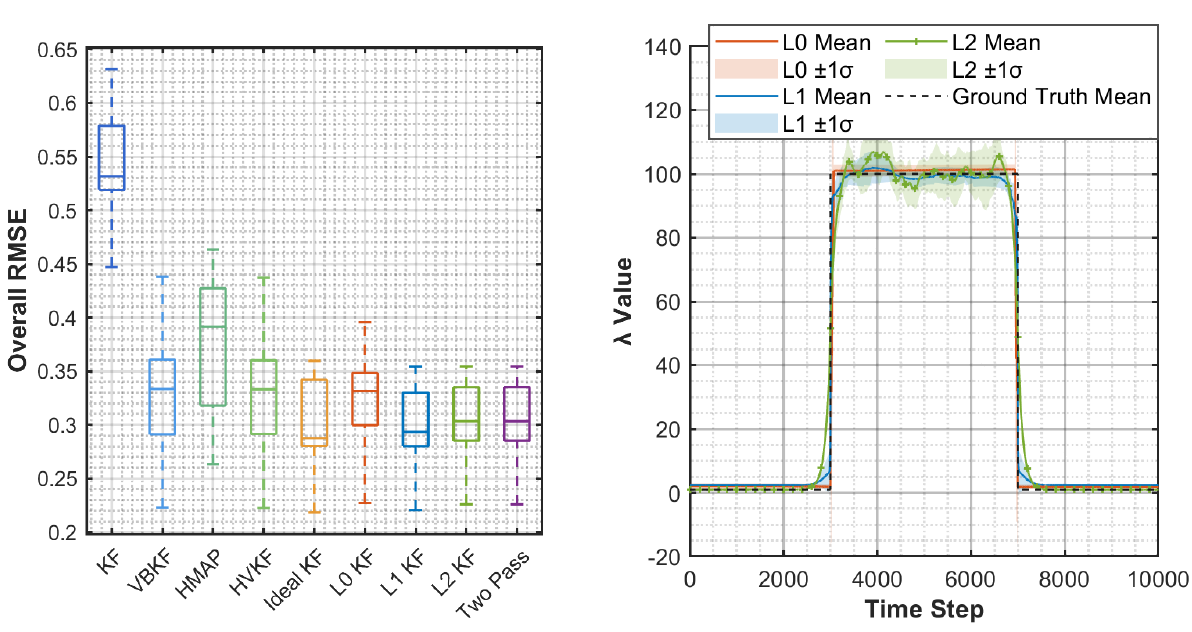}
		\caption{Performance comparison of different estimators. The left figure confirms the equivalent of $\ell_2$ regulation and two pass filter. The right figure show that the covariance estimator becomes a zero phase filter when containing the entire measurement sequence.}
		\label{REGHVKF}
	\end{figure}
	\subsection{Example 2: Rank Deficit Process Covariance Tracking}
	We consider a torsion load system \cite{luan2025kalman}. The system dynamics are governed by the following equations:
	\begin{equation}
		\begin{aligned}
			{x}_{k}&={A}_k {x}_{k-1}+{G}_{1,k} {u}_{k-1}+{G}_{2,k} {d}_{k-1}+{w}_{k}\\
			{y}_k&={C}_k {x}_{k}+{v}_{k}
			\label{linearfun}
		\end{aligned}
	\end{equation}
	with
	\begin{equation}\nonumber
		\tiny
		\begin{aligned}
			A_k&=\begin{bmatrix}
				0.9205 &   0.0795 &   0.0085 &   0.0003\\
				0.2045 &   0.7955 &   0.0007 &   0.0085\\
				-14.3468 &  14.3468 &   0.6872 &   0.0746\\
				37.5370 & -37.5370 &   0.1863 &   0.6405
			\end{bmatrix}\\
			G_{1,k}&=\begin{bmatrix}
				0.0826,
				0.0031,
				15.5568,
				1.2100
			\end{bmatrix}^{T}\\
			G_{2,k}&=\begin{bmatrix}
				0.0031,
				0.2076,
				1.2100,
				38.7470
			\end{bmatrix}^{T} \\
			C_k&=\begin{bmatrix}
				1&0&0&0\\
				0&1&0&0 \\
				0&0&1&0\\
				0&0&0&1
			\end{bmatrix}
		\end{aligned}
	\end{equation}
	where $x_k=[\theta_m,\theta_t,v_m,v_t]^{T} \in \mathbb{R}^{4}$ denotes the state vector, $u_k \in \mathbb{R}$ is the motor output torque, $d_k$ is the external disturbance, and $y_k$ is the measurement. 
	
	In the simulation, we set $w_k \sim \mathcal{N}(0, 0.01\mathrm{I}_4)$ and $v_k \sim \mathcal{N}(0, \mathrm{I}_4)$. Additionally, we model $d_k$ as Gaussian noise with $d_k \sim \mathcal{N}(0, 0.1\lambda_k)$, where $\lambda_k$ follows
	\begin{equation}
		\label{lambdak}
		\lambda_k = \begin{cases}
			1, & \text{if } k \in \mathbb{S}_1 \\
			100, & \text{if } k \in \mathbb{S}_2
		\end{cases}
	\end{equation}
	with $\mathbb{S}_1 = [0, 100) \cup [400, 600) \cup [900, 1000]$ and $\mathbb{S}_2 = [100, 400) \cup [600, 900)$. The time interval is $\Delta t=0.01$ s. We use the same nominal covariances for all estimators, except for the ground-truth covariance in Ideal KF. Meanwhile, we set $\nu=10$ for all process channels in STKF~\cite{a24}, apply a velocity-prior model with $\gamma=0.5$ to HOHVKF, and use a sliding window length of 5 for SWHVKF. Additionally, we employ the same $\rho=0.97$ for VBKF, HVKF, HOHVKF, and SWHVKF (specifically regarding all process channels for VBKF, and the $G_2 d_k$ channel in its lifted form). The overall accuracy of the different estimators is illustrated in Fig. \ref{QCompare}, with the covariance tracking visualization depicted in Fig. \ref{Qlambda} and the corresponding RMSE values detailed in Table \ref{tab:rmse_comparison}. The results demonstrate that the exact rank deficit process covariance estimation $\mathbb{E} [(G_{2,k}d_{k-1})(G_{2,k}d_{k-1})^{T}]$ can be tracked by the proposed methods (but is failed by the other estimators). Notably, over 100 Monte Carlo runs, the proposed HOHVKF and SWHVKF achieve performance almost identical to that of the Ideal KF (overall RMSE of $0.701 \pm 0.011$ for both HOHVKF and SWHVKF, compared to $0.070 \pm 0.011$ for the Ideal KF).
	
	\begin{figure}[!htp]
		\centering
		\includegraphics[width=0.9\columnwidth]{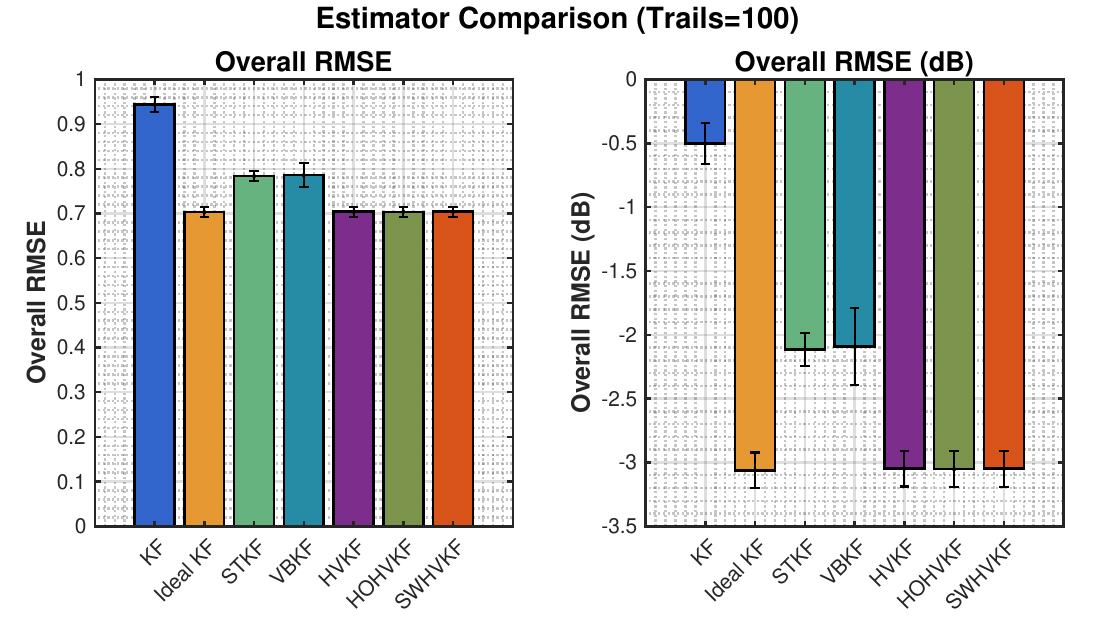}
		\caption{Comparison of different filtering algorithms over 100 Monte Carlo trials. The left panel displays the overall RMSE, while the right panel presents the overall RMSE in decibels (dB).}
		\label{QCompare}
	\end{figure}
	\begin{figure}[!htp]
		\centering
		\includegraphics[width=0.7\columnwidth]{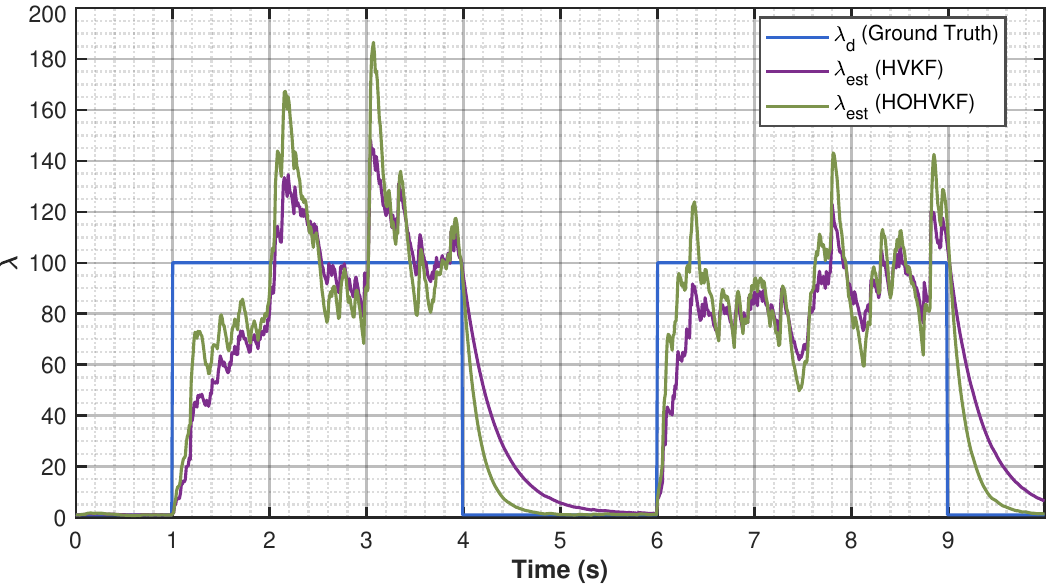}
		\caption{Tracking comparison for $\lambda$ (see \eqref{lambdak}) between the HVKF and HOHVKF algorithms. The higher-order approach (HOHVKF) yields a faster tracking response to sudden changes, at the cost of slightly increased fluctuation during the constant intervals.}
		\label{Qlambda}
	\end{figure}
	\begin{table*}[t]
		\centering
		\caption{Performance comparison of different estimators over 100 Monte Carlo trials (Values: mean $\pm$ std).}
		\scalebox{0.75}{
			\begin{tabular}{c|cccccc}
				\hline
				\hline
				\multirow{2}{*}{Estimators} & \multicolumn{4}{c}{RMSE of Individual States} & \multirow{2}{*}{Overall RMSE} & \multirow{2}{*}{ RMSE (dB)} \\
				\cline{2-5}
				& $x_1$ & $x_2$ & $x_3$ & $x_4$ & & \\
				\hline
				KF          &0.236 ± 0.015 & 0.236 ± 0.014 & 1.450 ± 0.039 & 1.146 ± 0.027 & 0.939 ± 0.020 &  -0.545 \\
				STKF        & 0.305 ± 0.023 & 0.425 ± 0.022 & 1.078 ± 0.025 & 1.001 ± 0.021 & 0.781 ± 0.013 &  -2.148 \\
				VBKF        & 0.288 ± 0.027 & 0.618 ± 0.141 & 0.980 ± 0.024 & 0.998 ± 0.021 & 0.781 ± 0.031 &  -2.146 \\
				HVKF        & 0.231 ± 0.015 & 0.231 ± 0.015 &0.930 ± 0.023 & 0.998 ± 0.021 & 0.702 ± 0.011 &  -3.079 \\
				{HOHVKF} & {0.231 ± 0.015} & {0.231 ± 0.015} & {0.929 ± 0.023} & {0.998 ± 0.021} & {0.701 ± 0.011} & {-3.085} \\
				{SWHVKF} & {0.231 ± 0.015} & {0.231 ± 0.015} & {0.930 ± 0.023} & {0.998 ± 0.021} & {0.701 ± 0.011} & {-3.080} \\
				\hline
				Ideal KF    & 0.231 ± 0.015 & 0.231 ± 0.015 & 0.927 ± 0.023 & 0.997 ± 0.021 & 0.700 ± 0.011 &  -3.094  \\
				\hline
				\hline
			\end{tabular}
		}
		\label{tab:rmse_comparison}
	\end{table*}
	\begin{table*}[h]
		\centering
		\caption{RMSEs (mean $\pm$ std) of different estimators over 100 Monte Carlo trials.}
		\label{tab:rmse_results_detailed}
				\scalebox{0.78}{
		\begin{tabular}{l c c c c c}
			\hline
			\hline
			Method & $d$ (Nm) & $\dot\theta$ (rad/s) & $\theta$ (rad) & tracking $\theta$ (rad) & tracking $\dot{\theta}$ (rad/s)\\
			\hline
			KF     & $5.4222 \pm 0.0795$ & $0.1748 \pm 0.0039$ & $0.0034 \pm 0.0002$ & $0.0246 \pm 0.0004$ & $0.8906 \pm 0.0101$ \\
			STKF   & $4.9861 \pm 0.0695$ & $0.0988 \pm 0.0026$ & $0.0033 \pm 0.0002$ & $0.0200 \pm 0.0004$ & $0.8306 \pm 0.0072$ \\
			VBKF   & $4.9755 \pm 0.0719$ & \cellcolor{gray!20}$0.0933 \pm 0.0023$ & $0.0033 \pm 0.0002$ & $0.0199 \pm 0.0004$ & $0.8301 \pm 0.0076$ \\
			HVKF   & $4.8888 \pm 0.0797$ & $0.0939 \pm 0.0024$ &\cellcolor{gray!20}$0.0031 \pm 0.0002$ &\cellcolor{gray!20} $0.0187 \pm 0.0004$ & $0.8296 \pm 0.0080$ \\
			HOHVKF &\cellcolor{gray!20} $\cellcolor{gray!20}4.8837 \pm 0.0799$ & \cellcolor{gray!20}$\cellcolor{gray!20}0.0933 \pm 0.0023$ & \cellcolor{gray!20}$\cellcolor{gray!20}0.0031 \pm 0.0002$ & \cellcolor{gray!20}$0.0187 \pm 0.0004$ & \cellcolor{gray!20}$0.8288 \pm 0.0080$ \\
			\hline
			\hline
		\end{tabular}}
	\end{table*}
	\subsection{Example 3: Disturbance Estimation}
	We consider a manipulator tracking example \cite{a9}. The augmented state-space model has
	\begin{equation}
		\begin{aligned}
			{x}_{k}&={\Phi}_k {x}_{k-1}+{G}_k {u}_{k-1}+B_d{w}_{d,k}+B_\theta{w}_{\theta,k}\\
			{y}_{k}&={H}_k{x}_{k}+B_v{v}_{k}
			\label{observer}
		\end{aligned}
	\end{equation}
	with
	\begin{equation}\nonumber
		\tiny
		\begin{aligned}
			{\Phi}_k&=\left[\begin{array}{ccc}
				1&0&0\\
				\frac{T}{{I}_{m}}& 1-\frac{b_{m}T}{{I}_{m}} &-\frac{k_{m}T}{{I}_{m}}\\
				0&T&1\\
			\end{array}\right],	{G}_k =\left[\begin{array}{c}
				0\\
				\frac{T}{{I}_{m}}\\
				0
			\end{array}\right]\\
			{B}_{d} &=\left[\begin{array}{c}
				1\\
				\frac{T}{I_m}\\
				0
			\end{array}\right],	{B}_{\theta} =\left[\begin{array}{c}
				0 \\
				T \\
				\frac{T^2}{2} 
			\end{array}\right],{B}_{v} =\left[\begin{array}{c}
				1 \\
				T \\
			\end{array}\right]\\
			{H}_k&=\left[\begin{array}{ccc}
				0 &1 &0\\
				0 &0 &0.1
			\end{array}\right]
		\end{aligned}
	\end{equation}
	where $x_{k}=[d_{k}, \dot{\theta}_{k},\theta_{k}]^{T}$ denotes
	the unknown external disturbance, the angular velocity and angle. $T=0.01$ s is the sampling interval, $u_k$ is the input torque driven by a computed torque controller (see \cite{a9}), $B_d$ and $B_\theta$ maps the disturbance noise $w_{d,k}$ and the process noise $w_{\theta,k}  \sim \mathcal{N}(0,0.01)$ into the systems, and $v_k \sim \mathcal{N}(0,0.01\mathrm{I}_2)$ is the measurement noise. In the simulation, we use ${I}_{m}=0.1$~Nm$\cdot$s$^2/\deg$, $b_m=1$~Nm$\cdot$s$/\deg$,  $k_m=0.1$~Nm$/\deg$. Moreover, we assume that the unknown disturbance has
	\begin{equation}\nonumber
		{d}_{k}=\left\{\begin{array}{l}
			50+{w}_{d1,k},400\le k \le 600 \\
			{w}_{d2,k}, \operatorname{otherwise}
		\end{array}\right.
	\end{equation}
	where $w_{d1,k}\sim \mathcal{N}(0,100)$ and $w_{d2,k}\sim \mathcal{N}(0,1)$. The desired angle follows $\theta_{d,k}=15 \sin (0.4\pi kT)$. 
	
	As for estimators, we apply the same nominal covariance for all estimators  (if applicable), and tune the hyperparameters (controlling the robustness level) for STKF, MKCKF~\cite{li2021multi} by trial and error. As for VBKF, HVKF, and HOHVKF, we use the same forgetting factor $\rho_d=0.96$, and apply $\gamma_d=0.5$ for HOHVKF. The estimation and control performance by utilizing different estimators is visualized in Fig. \ref{compare} and summarized in Table \ref{tab:rmse_results_detailed}, where the best values are highlighted in gray. We find the proposed methods are consistently better than KF, STKF, and VBKF. 
	
	\begin{figure}[!htp]
		\centering
		\includegraphics[width=0.90\columnwidth]{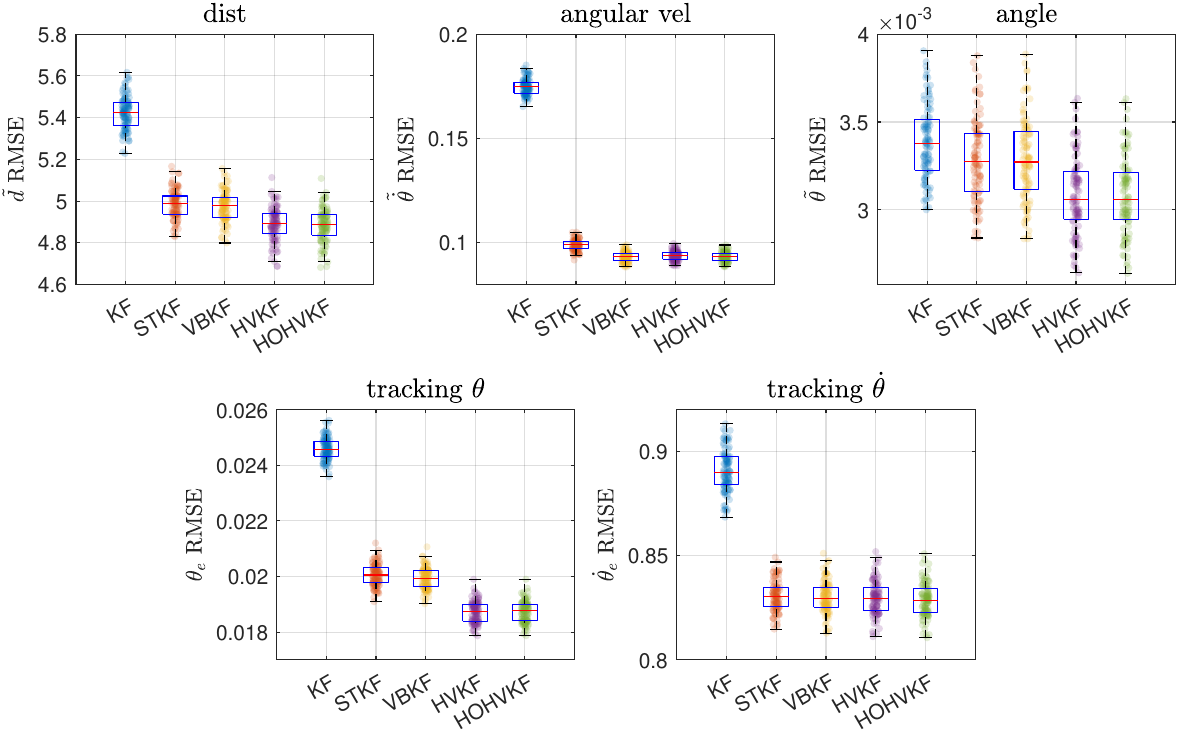}
		\caption{Performance comparison of KF, STKF, VBKF, HVKF, HOHVKF in 100 Monte Carlo runs. The sub-figures visualize the RMSEs of disturbance, angular velocity, angle, the angle tracking error, and the angular velocity tracking error, respectively.}
		\label{compare}
	\end{figure}
	
	\section{Conclusion}
	In this paper, we introduced a hierarchical variational Kalman filtering (HVKF) that systematically decouples variational inference into sequential robust state filtering and hyperparameter updating. By lifting the system dimensions, this framework unifies the treatment of process and measurement noise covariances, proving particularly advantageous for inferring rank-deficient process covariances. Capitalizing on this decoupled architecture, we successfully accelerated covariance convergence by either modifying the prior model of the latent scale parameters or estimating them via a sliding window mechanism. Extensive simulations confirm the superiority of the proposed approaches compared to existing approaches.
	\section{Appendix}
	\subsection{Proof of Lemma \ref{identlemma}}
	\label{identproof}
	The state $\chi_k$ is uniquely identifiable if and only if the null space of the design matrix $\mathcal{N}(W_k) = \{\mathbf{0}\}$. Let $q_k = \begin{bmatrix} m_k \\ x_k \end{bmatrix} \in \mathcal{N}(W_k)$, implying $W_k q_k = \mathbf{0}$. Expanding this expression yields:
	\begin{subequations}
		\label{eq:null_space_conditions}
		\begin{align}
			B_p^{-1} m_k &= \mathbf{0}, \label{eq:1} \\
			B^\dagger (m_k - x_k) &= \mathbf{0}, \label{eq:2} \\
			D^{-1} C x_k &= \mathbf{0}. \label{eq:3}
		\end{align}
	\end{subequations}
	According to \eqref{wreg} and \eqref{Bk}, we know that $B_p$ and $D$ are non-singular. Thus, \eqref{eq:1} guarantees $m_k = \mathbf{0}$. Substituting this result into \eqref{eq:2} gives $B^\dagger x_k = \mathbf{0}$. Note that $B^\dagger$ and $B^{T}$ share the same null space, since $B^\dagger= (B^T B)^{-1} B^T$ where $(B^T B)^{-1}$ is invertible. Thus, $x_k \in \mathcal{N}(B^T)$. Meanwhile, \eqref{eq:3} indicates $x_k \in \mathcal{N}(C)$ since $D$ is non-singular. Consequently, $\mathcal{N}(W_k) = \{\mathbf{0}\}$ if and only if $\mathcal{N}(B^T) \cap \mathcal{N}(C) = \{\mathbf{0}\}$.
	\subsection{The standard CAVI}
	\label{CAVIalg}
	We first focus on $q(\chi_k)$. Based on \eqref{vbsub1} and denoting $q(\chi_k) = \mathcal{N}(\mu_{\chi_k}, \Sigma_{\chi_k})$, the precision matrix determines the covariance and mean:
	\begin{equation}
		\Sigma_{\chi_k} = \left( W_k^\top \mathbb{E}\left[ \Lambda^{-1} \right] W_k \right)^{-1}, \quad \mu_{\chi_k} = \Sigma_{\chi_k} W_k^\top \mathbb{E}\left[ \Lambda^{-1} \right] t_k
		\label{vbchi}
	\end{equation}
	where $\mathbb{E}\left[ \Lambda^{-1} \right]$ is a diagonal matrix and 
	\begin{equation}\notag
		\mathbb{E}\left[ \Lambda^{-1} \right]_{ii}  = \frac{\nu_i + 1}{\nu_i \tau_i^2 + \mathbb{E}_{q(\chi_k)} \left[ (t_{k,i} - W_{k,i}\chi_k)^2 \right]}
		\label{wei}
	\end{equation}
	with $$\scalebox{0.85}{$\mathbb{E}_{q(\chi_k)} \left[ (t_{k,i} - W_{k,i}\chi_k)^2 \right] = (t_{k,i} - W_{k,i}\mu_{\chi_k})^2 + W_{k,i}\Sigma_{\chi_k}W_{k,i}^{T}$}.
	$$
	We then focus on $q(\lambda^{wv})$. Based on \eqref{vbsub2}, one can derive that $q^*(\lambda_i)$ follows an Inverse-Gamma distribution $\operatorname{IG}(\frac{\nu_i^{*}}{2}, \frac{\nu_i^{*}}{2}(\tau_i^{*})^{2})$ with 
	$$
	\nu_i^{*} = \nu_i + 1,~(\tau_i^{*})^2 = \frac{\nu_i \tau_i^2 + \mathbb{E}_{q(\chi_k)} \left[ (t_{k,i} - W_{k,i}\chi_k)^2 \right]}{\nu_i + 1}.
	$$
	\subsection{Proof of Theorem \ref{convRWLS}}
	\label{fixproof}
	Let $e_i(\chi_k) = [t_k - W_k\chi_k]_i$ denote the $i$-th residual element. The objective function in \eqref{rbobj} can be reformulated as $\mathcal{J}(\chi_k) = \sum_{i} \phi(e_i^2(\chi_k))$, where the scalar function $\phi : \mathbb{R}_{\ge 0} \to \mathbb{R}$ is defined as $\phi(u) = \frac{\nu_i+1}{2} \log(1 + \frac{u}{\nu_i\tau_i^2})$. Since $\phi(u)$ is continuously differentiable and strictly concave on $u \ge 0$ (verified by $\phi''(u) < 0$), it is bounded above by its first-order Taylor expansion at any expansion point $u_0$. Specifically, 
	$$\phi(u) \le \phi(u_0) + \phi'(u_0)(u - u_0).$$
	Let $u_i^{(j)} \triangleq (e_i(\chi_k^{(j)}))^2$ be the squared residual at iteration $j$. Applying the concavity property, we construct a surrogate function $Q(\chi_k, \chi_k^{(j)})$ that serves as a global upper bound for $\mathcal{J}(\chi_k)$:
	\begin{equation}
		\scalebox{0.8}{$
			\mathcal{J}(\chi_k) \le Q(\chi_k, \chi_k^{(j)}) := \sum_{i} \left[ \phi(u_i^{(j)}) + \phi'(u_i^{(j)}) \left( e_i^2(\chi_k) - u_i^{(j)} \right) \right].$}
		\label{eq:surrogate}
	\end{equation}
	Minimizing $Q(\chi_k, \chi_k^{(j)})$ with respect to $\chi_k$ is equivalent to minimizing the weighted least squares term $\sum_{i} \omega_i^{(j)} e_i^2(\chi_k)$, as the other terms are constant with respect to $\chi_k$. The weights are given by $\omega_i^{(j)} = \phi'(u_i^{(j)}) = \frac{\nu_i+1}{\nu_i\tau_i^2 + u_i^{(j)}}$, which is identical to the diagonal elements of $\Lambda_k(e_k)$ defined in \eqref{weight}. The fixed-point iteration \eqref{solu1} computes the global minimizer of this surrogate:
	\begin{equation}
		\chi_k^{(j+1)} = \arg\min_{\chi_k} Q(\chi_k, \chi_k^{(j)}).
	\end{equation}
	The monotonic descent property follows from the majorization-minimization principle:
	\begin{equation}
		\mathcal{J}(\chi_k^{(j+1)}) \le Q(\chi_k^{(j+1)}, \chi_k^{(j)}) \le Q(\chi_k^{(j)}, \chi_k^{(j)}) = \mathcal{J}(\chi_k^{(j)}).
	\end{equation}
	Since $\mathcal{J}(\chi_k)$ is bounded below by 0, the sequence of objective values $\{\mathcal{J}(\chi_k^{(j)})\}$ converges. This completes the proof.
	
	Some representative losses that can be solved by RWLS with guaranteed convergence are listed in Table \ref{lossRWLS}.
	\begin{table}[h]
		\centering
		\caption{Exemplary losses that can be solved by RWLS.}
		\scalebox{0.85}{
			\begin{tabular}{cc}
				\hline
				\hline 
				{Robust Loss} & {Function $\rho(x)$} \\
				\hline
				Geman-McClure & $\frac{c^2 x^2}{2(c^2 + x^2)}$ \\
				Welsch & $c^2 \left(1 - \exp\left(-\frac{x^2}{2c^2}\right)\right)$\\
				Smoothed Huber & $c^2 \left(\sqrt{1 + (x/c)^2} - 1\right)$ \\
				Fair & $c^2 \left( \frac{|x|}{c} - \ln\left(1 + \frac{|x|}{c}\right) \right)$ \\
				\hline 
				\hline                                         
		\end{tabular}}
		\label{lossRWLS}
	\end{table}
	\subsection{An Alternative Expression of \eqref{update} and \eqref{errcov}}
	\label{alternate}
	By expanding $\tilde{P}_{\chi k}^{-}$ in  \eqref{update}, we arrive at
	\begin{equation}
		\tiny
		\begin{aligned}
			x_{k,t}^{+} &= x_k^{-} + \tilde{K}_{xk,t}(y_k - C x_k^{-}) \\
			\tilde{K}_{xk,t} &= \tilde{P}_{k}^{-} C^{T}(C \tilde{P}_{k}^{-} C^{T} + \tilde{R}_{k})^{-1} \\
			m_{k,t}^{+} &= m_k^{-} + \tilde{K}_{mk,t}(y_k - C x_k^{-}) \\
			\tilde{K}_{mk,t} &= \Sigma_{Pk} C^{T}(C \tilde{P}_{k}^{-} C^{T} + \tilde{R}_{k})^{-1} \\
			\tilde{P}_{k}^{-} &= \Sigma_{Pk} + \tilde{Q}_{k} \\
			\tilde{Q}_{k} &= B \Lambda_{w}^{-1} B^{T} \\
			\Lambda_{w} &= \operatorname{diag}[d_{\nu_{n+1}}(e_{n+1}), \ldots, d_{\nu_{n+p}}(e_{n+p})]\\
			\tilde{R}_{k} &= D \Lambda_{v}^{-1} D^{T}
			\\
			\Lambda_{v} &= \operatorname{diag}[d_{\nu_{n+p+1}}(e_{n+p+1}), \ldots, d_{\nu_{n+p+m}}(e_{n+p+m})] \\
			e_k &= t_k - W_k \chi_k^{t-1}
		\end{aligned}
		\label{update1}
	\end{equation}
	Correspondingly, at the final iteration, the posteriori error covariance can be update as
	\begin{equation}
		\tiny
		\begin{aligned}
			P_{xk}^{+} &= (I - \tilde{K}_{xk,t} C) {P}_{k}^{-} (I - \tilde{K}_{xk,t} C)^T + \tilde{K}_{xk,t} {R}_{k} \tilde{K}_{xk,t}^{T}\\
			P_{m_k}^{+} &= (I - \tilde{K}_{mk,t} C) \Sigma_{Pk} (I - \tilde{K}_{mk,t} C)^T + \tilde{K}_{mk,t} {R}_{k} \tilde{K}_{mk,t}^T\\
			P_{m_k x_k}^+ &=(P_{ x_k x_k}^+)^{T}= \Sigma_{Pk} (I - C^T \tilde{K}_{xk,t}^T)
		\end{aligned}
		\label{posterior_covariance}
	\end{equation}
	
	\subsection{Proof of Proposition \ref{scaleinv}}
	\label{pscale}
	We proceed by induction. For the base case $k=0$, $\hat{x}'_{0|0} = \hat{x}_{0|0}$ and $P'_{0|0} = \alpha P_0$ hold by definition. Assume $\hat{x}'_{k-1|k-1} = \hat{x}_{k-1|k-1}$ and $P'_{k-1|k-1} = \alpha P_{k-1|k-1}$. The recursive filter updates yield:$$
	\tiny
	\begin{aligned}
		\hat{x}'_{k|k-1} &= A\hat{x}'_{k-1|k-1} = A\hat{x}_{k-1|k-1} = \hat{x}_{k|k-1} \\
		P'_{k|k-1} &= A(\alpha P_{k-1|k-1})A^T + \alpha Q_{k-1} = \alpha P_{k|k-1} \\
		K'_k &= (\alpha P_{k|k-1})C^T \big[ C(\alpha P_{k|k-1})C^T + \alpha R_k \big]^{-1} = K_k \\
		\hat{x}'_{k|k} &= \hat{x}'_{k|k-1} + K'_k(y_k - C\hat{x}'_{k|k-1}) = \hat{x}_{k|k} \\
		P'_{k|k} &= (I - K'_k C)P'_{k|k-1} = (I - K_k C)(\alpha P_{k|k-1}) = \alpha P_{k|k}
	\end{aligned}$$
	By induction, the lemma holds for all $k \ge 1$.
	\subsection{Proof of Theorem \ref{prop:refined_estimator}}
	\label{prefined_estimator}
	Taking the expectation of \eqref{scale_update} yields a recursive geometric series that asymptotically converges to $\sigma_i^2$, satisfying Property 1. For Property 2, assuming temporal independence of $\mathcal{E}_{i,k}$, by taking variance of \eqref{scale_update} yields
	$$
	\operatorname{Var}[(\tau_{i,k|k})^2] = \rho_i^{2} \operatorname{Var}[(\tau_{i,k-1|k-1})^{2}] + (1-\rho_i)^2 \operatorname{Var}[ \mathbb{E}[e_i^2]].
	$$
	Subsequently, the steady-state variance of the EWMA filter is given by $V_{\infty} = \frac{1-\rho_i}{1+\rho_i} \eta_i^4$. Substituting $\rho_i = \frac{\nu_i}{\nu_i+1}$ gives
	\begin{equation*}
		V_{\infty} = \frac{1 - \frac{\nu_i}{\nu_i+1}}{1 + \frac{\nu_i}{\nu_i+1}} \eta_i^4 = \frac{1}{2\nu_i + 1} \eta_i^4.
	\end{equation*}
	For Property 3, the EWMA time constant is $T_i = \frac{-1}{\ln(\rho_i)}$ (by the definition of EWMA).
	\subsection{Proof of Theorem \ref{prop:sliding_smoothing_equivalence}}
	\label{psliding}
	Setting the gradient $\partial \mathcal{J} / \partial \tau_{i,j}^2 = 0$ yields the symmetric tridiagonal stationarity condition for the interior points:
	\begin{equation} \label{eq:stationarity_window}
		(1+2\lambda)(\tau_{i,j|k})^2 - \lambda (\tau_{i,j-1|k})^2 - \lambda (\tau_{i,j+1|k})^2 = \mathcal{E}_{i,j}.
	\end{equation}
	Substituting $\lambda = \frac{\rho_i}{(1-\rho_i)^2}$ and multiplying by $(1-\rho_i)^2$, this condition becomes:
	\begin{equation} \label{eq:stationarity_sub_window}
		(1+\rho_i^2)(\tau_{i,j|k})^2 - \rho_i (\tau_{i,j-1|k})^2 - \rho_i (\tau_{i,j+1|k})^2 = (1-\rho_i)^2 \mathcal{E}_{i,j}.
	\end{equation}
	To verify the proposed two-pass solution, we isolate the forward estimate from the backward pass:
	\begin{equation} \label{eq:isolate_forward}
		(1-\rho_i)(\tau_{i,j|j})^2 = (\tau_{i,j|k})^2 - \rho_i (\tau_{i,j+1|k})^2.
	\end{equation}
	Rearranging the forward filter gives $(1-\rho_i)\mathcal{E}_{i,j} = (\tau_{i,j|j})^2 - \rho_i (\tau_{i,j-1|j-1})^2$. Multiplying by $(1-\rho_i)$ and substituting the relation from \eqref{eq:isolate_forward} for both steps $j$ and $j-1$ yields:
	\begin{equation}
		\begin{aligned}
			(1-\rho_i)^2 \mathcal{E}_{i,j} &= \left[ (\tau_{i,j|k})^2 - \rho_i (\tau_{i,j+1|k})^2 \right] \\
			&- \rho_i \left[ (\tau_{i,j-1|k})^2 - \rho_i (\tau_{i,j|k})^2 \right].
		\end{aligned}
	\end{equation}
	Expanding the right side perfectly recovers the global optimality condition \eqref{eq:stationarity_sub_window}. Because $\mathcal{J}$ is strictly convex, satisfying this stationarity condition guarantees that the two-pass sequence is the unique exact minimizer.
	
	\bibliographystyle{IEEEtranN}
	\bibliography{ReferenceVector}
	
	\begin{wrapfigure}{l}{1.8cm}
		\includegraphics[width=0.9in,height=1.25in,clip,keepaspectratio]{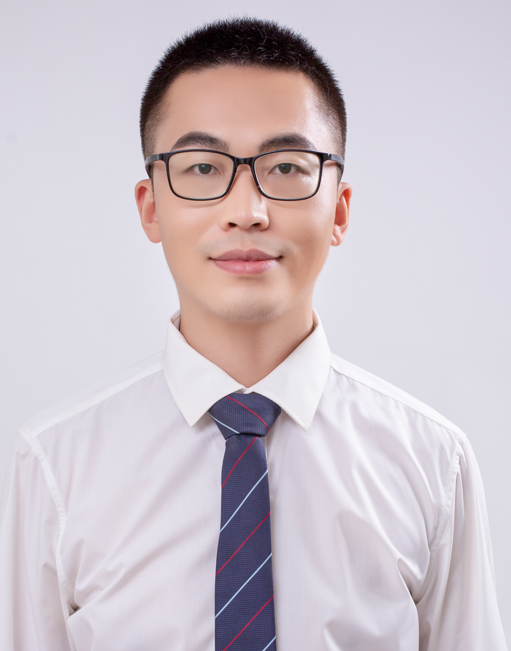}
	\end{wrapfigure}
	Shilei Li received the B.E. degree in Detection Guidance and Control Technology and M.S. degree in Control Engineering both from Harbin Institute of Technology, Harbin, China, in 2015 and 2018 respectively, and the Ph.D. degree in Electronic and Computer Engineering from The Hong Kong University of Science and Technology, Hong Kong, China, in 2022. Currently, he is an assistant professor at the School of Automation, Beijing Institute of Technology, China. His research interests include robust filtering and smoothing, human-robot interaction, and exoskeleton robots.

	\begin{wrapfigure}{l}{1.8cm}
		\includegraphics[width=0.9in,height=1.25in,clip,keepaspectratio]{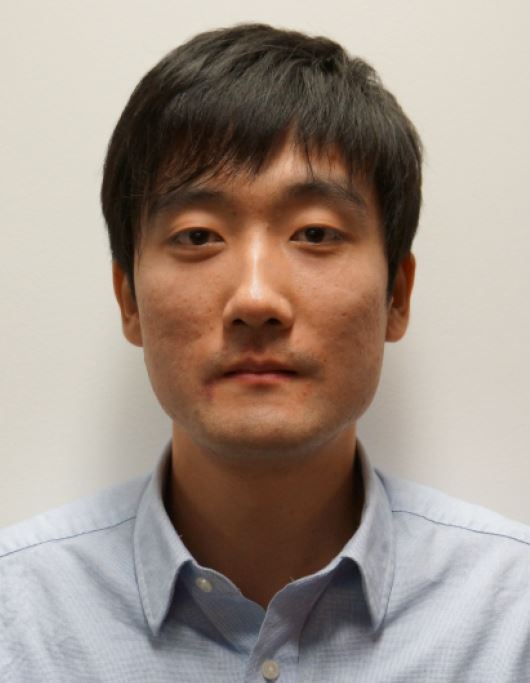}
	\end{wrapfigure}
	Dawei Shi  received the B.E. degree in Electrical Engineering and Automation from the Beijing Institute of Technology, Beijing, China,
	in 2008, the Ph.D. degree in Control Systems from the University of Alberta, Edmonton, AB, Canada, in 2014. In December 2014, he was appointed as an Associate Professor at the School of Automation, Beijing Institute of Technology. From February 2017 to July 2018, he was with
	the Harvard John A. Paulson School of Engineering and Applied Sciences, Harvard University, as a Postdoctoral Fellow in bioengineering. Since July 2018, he has been with the School of Automation, Beijing Institute of Technology, where he is a professor. His research focuses on the analysis and synthesis of complex sampled-data control systems with applications to biomedical engineering, robotics, and motion systems. He serves as an Associate Editor/Technical Editor for IEEE Transactions on Industrial Electronics, IEEE/ASME Transactions on Mechatronics, IEEE Control Systems Letters, and IET Control Theory and Applications. He is a member of the Early Career Advisory Board of Control Engineering Practice. He was a Guest Editor for European Journal of Control. He served as an associate editor for IFAC World Congress and is a member of the IEEE Control Systems Society Conference Editorial Board. He is a Senior Member of the IEEE.

	\begin{wrapfigure}{l}{1.8cm}
		\includegraphics[width=0.9in,height=1.25in,clip,keepaspectratio]{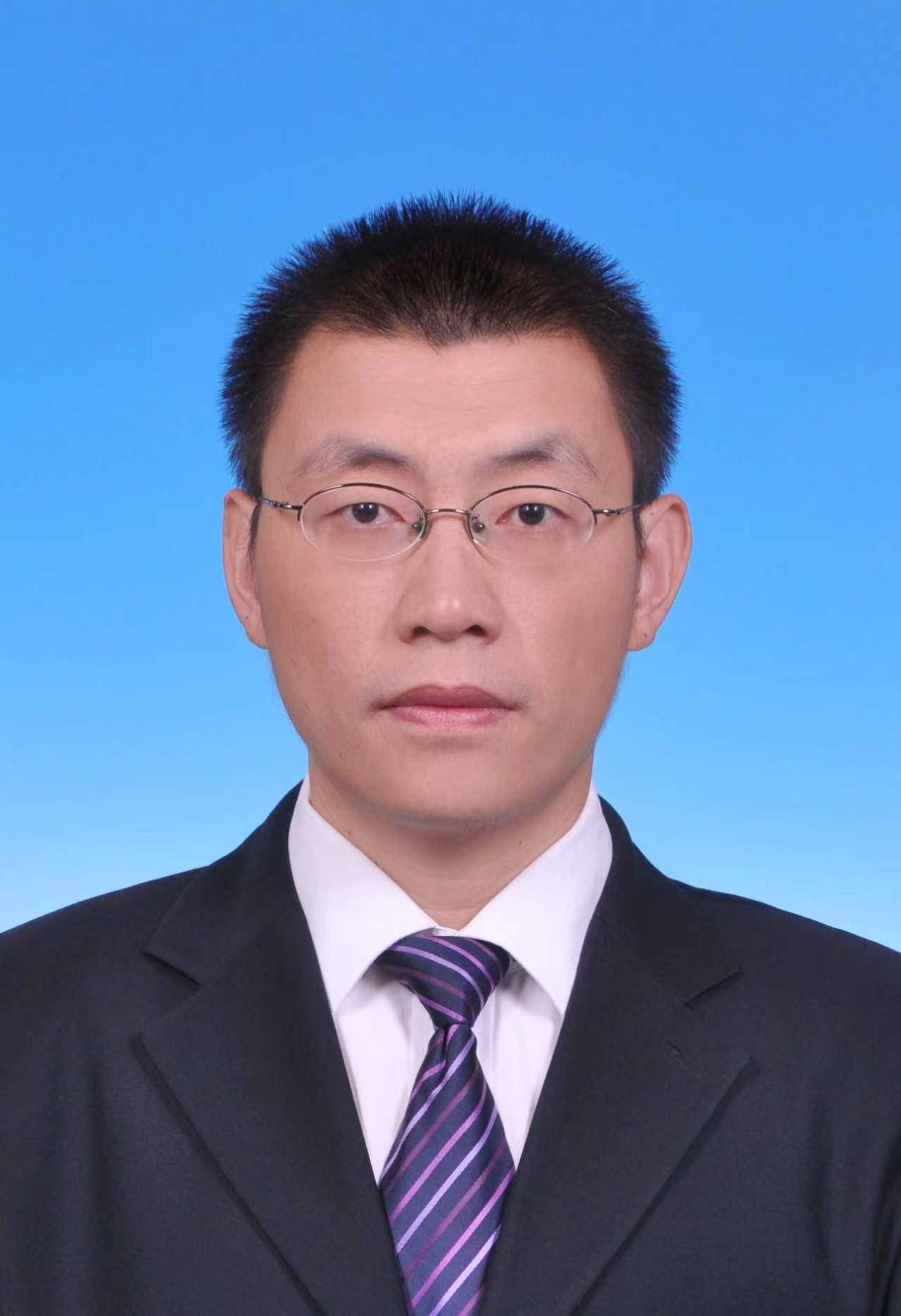}
	\end{wrapfigure}
	Wei Zheng received his Bachelor of Engineering degree in Electronic Information Engineering from the Department of Electronic Engineering, Tsinghua University in 2002. He obtained his master’s degree in 2008 and his Ph.D. in 2013, both from the Intelligent Information Processing Laboratory, Institute of Computing Technology, Chinese Academy of Sciences. From 2013 to March 2016, he conducted postdoctoral research at the Visual Intelligence Laboratory of Samsung Research Institute (Beijing) Communication Technology.
	His research findings have been published in numerous renowned international conferences and journals, including CVPR, ACCV, IEEE Transactions, and Elsevier, and he holds a number of domestic and international invention patents.
	From 2009 to 2010, he conducted research on object recognition and multimedia retrieval at Panasonic Singapore Laboratories and Microsoft Research Asia (MSRA) respectively. From 2013 to 2017, he led the development of intelligent algorithms for mobile devices at Samsung Research Institute Beijing Communication Technology, where he led his team in upgrading recognition algorithms from traditional classification methods to deep learning approaches, ultimately delivering core algorithms for the intelligent photo album management engine in Samsung flagship smartphones.
	From 2017 to 2019, he served as an algorithm researcher at the joint laboratory of the Hong Kong University of Science and Technology and Youjia Innovation Co., Ltd., focusing on visual recognition algorithms for intelligent driving. Since 2019, he has been Vice President of Algorithm R\&D at Youjia Innovation Co., Ltd., leading a R\&D team of over 100 members to develop intelligent driving algorithms and designing core algorithm architectures for multiple autonomous driving products of the company.
	
	\begin{wrapfigure}{l}{1.8cm}
		\includegraphics[width=0.9in,height=1.25in,clip,keepaspectratio]{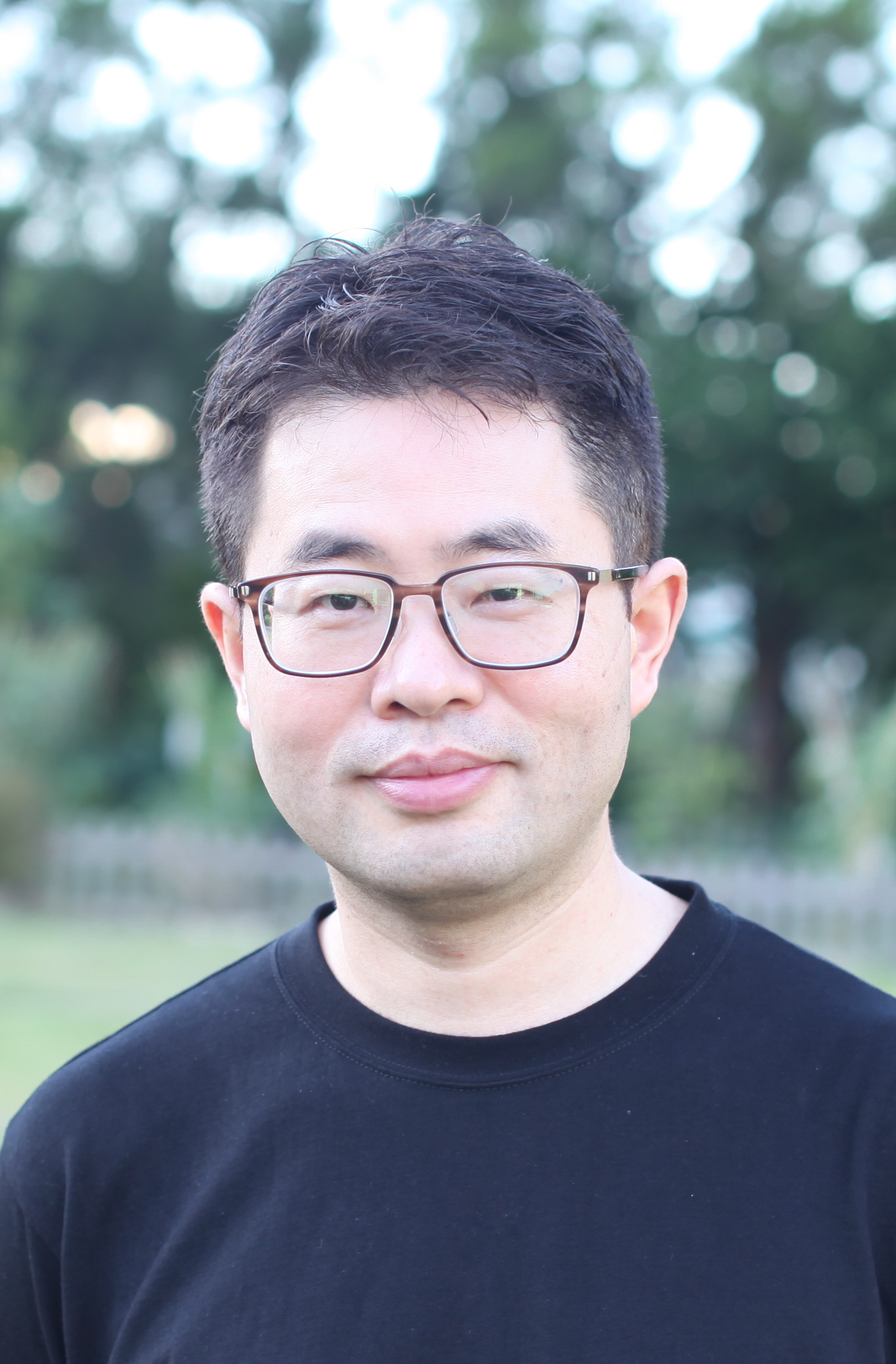}
	\end{wrapfigure}
	Ling Shi(Fellow, IEEE) received the B.S. degree in electrical and electronic engineering from Hong Kong University of Science and Technology, Kowloon, Hong Kong, in 2002 and the Ph.D. degree in Control and Dynamical Systems from California Institute of Technology, Pasadena, CA, USA, in 2008. He is currently a Professor in the Department of Electronic and Computer Engineering, and the associate director of the Robotics Institute, both at the Hong Kong University of Science and Technology. His research interests include cyber-physical systems security, networked control systems, sensor scheduling, event-based state estimation, and exoskeleton robots. He served as an editorial board member for the European Control Conference 2013-2016. He was a subject editor for International Journal of Robust and Nonlinear Control (2015-2017). He has been serving as an associate editor for IEEE Transactions on Control of Network Systems from July 2016, and an associate editor for IEEE Control Systems Letters from Feb 2017. He also served as an associate editor for a special issue on Secure Control of Cyber Physical Systems in the IEEE Transactions on Control of Network Systems in 2015-2017. He served as the General Chair of the 23rd International Symposium on Mathematical Theory of Networks and Systems (MTNS 2018). He is a member of the Young Scientists Class 2020 of the World Economic Forum (WEF) and he is a fellow of IEEE (Class 2023). 
	
\end{document}